\documentclass[11pt]{amsart}
\usepackage[margin=1in]{geometry}
\usepackage{bbm}
\usepackage{float, graphicx}
\usepackage[]{epsfig}
\usepackage{amsmath, amsthm, amssymb,}
\usepackage{epsfig}
\usepackage{verbatim}
\usepackage{multicol}
\usepackage{url}
\usepackage{latexsym}
\usepackage{mathrsfs}
\usepackage[colorlinks, bookmarks=true]{hyperref}
\usepackage{graphicx}
\usepackage{amsmath}
\usepackage{enumerate}
\usepackage[normalem]{ulem}
\usepackage{bm}
\usepackage{dsfont}
\usepackage{stmaryrd}
\usepackage{enumitem}
\usepackage{soul}
\usepackage[noadjust]{cite}
\usepackage{color}
\usepackage{xcolor}
\usepackage{hyperref}
\usepackage{indentfirst}
\usepackage[capitalise]{cleveref}

\usepackage{mathtools}
\newcommand{\trace}{\mathrm{trace}}

\newcommand{\cL}{{\mathcal L}}

\newcommand{\la}{{\lambda}}

\newcommand{\cD}{{\mathcal D}}

\newcommand{\cN}{{\mathcal N}}

\newcommand{\bR}{{\mathbb R}}

\newcommand{\bE}{{\mathbb E}}

\def\TV{\mathsf{TV}}

\def\<{\langle}
\def\>{\rangle}

\def\Lamp[#1]{\boldsymbol{\Lambda}_{\mathrm{AMP}}^{(#1)}}
\def\lalg[#1]{\Lambda_{\mathrm{alg}, #1}}

\def\de{{\rm d}}

\def\RR{\mathbb{R}}

\newcommand{\rd}{{\rm d}}

\newcommand{\del}{\partial}
\newcommand{\al}{\alpha}

\def\cO{\mathcal{O}}

\def\id{{\mathbb I}}

\def\cL{{\mathcal L}}

\newcommand{\E}{\mathbb{E}}

\newcommand{\N}{\mathcal{N}}

\newcommand{\sign}{\operatorname{sign}}

\newcommand{\RN}[1]{%
  \textup{\uppercase\expandafter{\romannumeral#1}}%
}

\newcommand{\RNum}[1]{\uppercase\expandafter{\romannumeral #1\relax}}

\newcommand\qU{U}
\newcommand\svarrho{\widehat{\varrho}}
\newcommand\sq{\widehat{q}}
\newcommand\sU{\widehat{U}}
\newcommand\sV{\widehat{V}}
\newcommand\sY{\widehat{Y}}
\newcommand\dq{\widetilde{q}}
\newcommand\dV{\widetilde{V}}
\newcommand\dU{\widetilde{U}}
\newcommand\dY{\widetilde{Y}}
\newcommand\dvarrho{\widetilde{\varrho}}

\newcommand\muast{\mu_{\ast}}
\newcommand\mft{\tau}
\theoremstyle{plain} %plain, definition, remark
\newtheorem{theorem}{Theorem}[section]
\newtheorem*{theorem*}{Theorem}
\newtheorem{lemma}[theorem]{Lemma}
\newtheorem*{lemma*}{Lemma}

\newtheorem*{corollary*}{Corollary}
\newtheorem{proposition}[theorem]{Proposition}
\newtheorem*{proposition*}{Proposition}
\newtheorem{assumption}[theorem]{Assumption}
\newtheorem*{assumption*}{Assumption}

\newtheorem*{definition*}{Definition}
\newtheorem{example}[theorem]{Example}
\newtheorem*{example*}{Example}
\newtheorem{remark}[theorem]{Remark}

\newtheorem*{remark*}{Remark}
\newtheorem*{remarks*}{Remarks}

\title{Convergence Analysis of Probability Flow ODE for Score-based Generative Models}
\date{}

\begin{document}

\author{Daniel~Zhengyu~Huang\textsuperscript{1}}
\address{\textsuperscript{1}Beijing International Center for Mathematical Research, Center for Machine Learning Research, Peking University, Beijing, China}
\email{huangdz@bicmr.pku.edu.cn}
\author{Jiaoyang Huang\textsuperscript{2}}
\address{\textsuperscript{2}Department of Statistics and Data Science at University of Pennsylvania, Philadelphia, PA, USA}
\email{huangjy@wharton.upenn.edu}
\author{Zhengjiang Lin\textsuperscript{3}}
\address{\textsuperscript{3}Department of Mathematics at Massachusetts Institute of Technology, Cambridge, MA, USA}
\email{linzj@mit.edu}

\begin{abstract}
Score-based generative models have emerged as a powerful approach for sampling high-dimensional probability distributions. Despite their effectiveness, their theoretical underpinnings remain relatively underdeveloped. In this work, we study the convergence properties of deterministic samplers based on probability flow ODEs from both theoretical and numerical perspectives. Assuming access to $L^2$-accurate estimates of the score function, we prove the total variation between the target and the generated data distributions can be bounded above by $\mathcal{O}(d^{3/4}\delta^{1/2})$ in the continuous time level, where  $d$ denotes the data dimension and $\delta$ represents the $L^2$-score matching error. For practical implementations using a $p$-th order Runge-Kutta integrator with step size $h$, we establish error bounds of $\mathcal{O}(d^{3/4}\delta^{1/2} + d\cdot(dh)^p)$ at the discrete level. Finally, we present numerical studies
on problems up to 128 dimensions to verify our theory.
\end{abstract}

\maketitle
\section{Introduction}

In recent years, score-based generative models \cite{sohl2015deep,ho2020denoising, song2019generative,song2020score,dhariwal2021diffusion} have emerged as a powerful paradigm for sampling high-dimensional probability distributions. Unlike traditional generative models that directly parameterize the mapping from random noise to target distribution samples~\cite{kingma2013auto,goodfellow2014generative,rezende2015variational,papamakarios2021normalizing}, score-based generative models
consist of two stochastic processes---the forward and reverse processes. 
The forward process transforms samples from the target data distribution $\muast$  with density $q_0$ into pure noise, a step commonly referred to as the diffusion process.  
The gradient of the log-density function, also known as the score function, is learned from these trajectories using score matching techniques~\cite{hyvarinen2005estimation,vincent2011connection,song2019generative,song2020score}.
The reverse process, guided by the score function, transforms random noise back into samples from $q_0$.
This methodology has been proven effective in synthesizing high-fidelity audio and image data~\cite{dhariwal2020jukebox,dhariwal2021diffusion,popov2021grad,ramesh2022hierarchical,esser2024scaling}.

% Introduce 

The reverse process is commonly implemented either as stochastic dynamics or deterministic dynamics, the latter often formulated as probability flow ordinary differential equation (ODE). 
These probability flow ODEs can typically be discretized using numerical methods such as forward Euler, exponential integrator, Heun, and high-order Runge-Kutta methods. 
Recent advancements in methods like those proposed in \cite{song2020denoising,song2020score,lu2022dpm,lu2022dpm+,zhang2022fast,zhao2024unipc} have enabled denoising steps to be completed in just a few iterations (e.g., 50 steps), compared to the Euler-Maruyama scheme typically employed for stochastic dynamics, which often requires a significantly larger number of steps (e.g., 1000 steps). Consequently, these deterministic methods achieve better efficiency in generating samples with only moderate quality degradation.
The deterministic dynamics depends on the score function, which is typically learned by a neural network through the score matching process involving non-convex optimization. Consequently, the score estimation is inherently imperfect. 
This imprecision, coupled with discretization error, poses a critical question: 
\ul{How does the interplay between score matching error and discretization error influence the convergence of the deterministic dynamics towards the true data distribution?} 
Our work seeks to address this question by delving into the convergence analysis of probability flow ODEs within the context of score-based generative models.

For the stochastic dynamics,  convergence analyses have been explored in various works such as \cite{block2020generative,de2021diffusion, yang2022convergence,kwon2022score,tang2024contractive,lee2022convergence, chen2022sampling,lee2023convergence,chen2023score,chen2023improved,wu2024theoretical,tang2024score,mooney2024global}, with notable contributions from~\cite{chen2022sampling,lee2023convergence,chen2023improved,tang2024score},
offering convergence guarantees with
polynomial complexity, without relying on any structural
assumptions on the data distribution like log-concavity. The stochastic nature of these dynamics plays a crucial role in mitigating error accumulation.
However, the deterministic counterpart warrants further exploration. 
Related works include \cite{chen2023restoration}, which assumes no score matching error and
provides a discretization analysis for the probability flow ODE in KL divergence. However, their bounds exhibit a large dependence on dimensionality and are exponential in the Lipschitz constant of the score integrated over time.
In contrast, \cite{chen2024probability} assumes $L^2$ bounds on the score estimation and offers polynomial-time convergence guarantees for the probability
flow ODE combined with a stochastic Langevin corrector, without relying on any structural assumptions on the data distribution.
Similarly, \cite{li2024accelerating,li2023towards} 
provide polynomial-time convergence guarantees for the probability flow ODE by requiring control of the difference
between the derivatives of the true and approximate scores.
Additionally, \cite{albergo2022building,benton2023error} analyze the convergence of the deterministic dynamics at a continuous time level, exhibiting exponential dependence on the Lipschitz constant, stemming from a more general stochastic interpolant or flow matching setup~\cite{lipman2022flow,liu2022flow,albergo2023stochastic,boffi2206probability}. Finally,
\cite{gao2024convergence}  offers convergence analysis for the general probability flow ODEs with log-concavity data assumption, where the error bounds grow exponentially with time $T$ in the presence of the score matching error.

\subsection{Our Contributions}
We analyze the convergence of the probability flow ODE from both theoretical and numerical perspectives. Our detailed contributions are as follows:
\begin{itemize}
    \item We provide convergence guarantees of the probability flow ODE at the continuous time level under three mild assumptions. These assumptions are as follows: \cref{assumption:secon-moment} asserts that the target density has compact support, \cref{a:score-estimate} asserts the $L^2$ score matching error over time is bounded by $\delta$, and  \cref{a:score-derivative} asserts the first and second derivatives of the estimated score are bounded. Under these assumptions, we prove in \cref{theorem: main L^1 theorem} that the total variation distance between the target and the generated data distributions can be bounded above by $\cO(d^{3/4} \delta^{1/2})$, where  $d$ is the data dimension.
    \item We provide convergence guarantees of the probability flow ODE at the discretized level. To accommodate a $p$-th order time integrator, we further require \cref{a:score-high-derivative}, that the estimated score function's first $(p+1)$-th derivatives are bounded. We establish in \cref{theorem: main L^1 theorem discretized} that the total variation distance between the target and the generated data distributions can be bounded above by $\mathcal{O}(d^{3/4}\delta^{1/2} + d\cdot(dh)^p)$. This implies an iteration complexity bound $\cO(d^{1+1/p} \varepsilon^{-1/p})$ for achieving a total variation accuracy of $\varepsilon$.

    \item  We verify our theoretical discoveries through numerical studies on problems with Gaussian mixture target densities up to 128 dimensions.  By intentionally introducing artificial score matching errors and employing the widely used second-order Heun's time integrator, our numerical results demonstrate a total variation error of $\mathcal{O}(\delta + h^2)$ (for the marginal distributions). The quadratic dependence of $h$ matches our theory with $p=2$.
    
\end{itemize}

In our theoretical proof at the continuous time level, we combine the method of characteristic lines and calculus of variations to estimate the total variation between the generated data distribution  $\widehat q_t$ and the target distribution $q_t$ along the diffusion process. Compared to using Gr\"{o}nwall's inequality directly, our error estimate in Theorem~\ref{theorem: L^1 error} does not include an exponential term in time. We provide two mathematically rigorous yet simple proofs of Theorem~\ref{theorem: L^1 error}, and also illustrate our intuition in Section~\ref{section: L^1 error of transport equation}. Furthermore, our methods imply a more general Theorem~\ref{theorem: general L^1 error} for the $L^1$-norms of solutions of general transport equations. In Remark~\ref{remark: main W^p,r error small} and Remark~\ref{remark: general W^k,1 error}, we highlight that our methods can also estimate the $L^1(\mathbb{R}^d)$-norms of derivatives of $\widehat q_t-q_t$. Consequently, we can conclude that the $L^r(\mathbb{R}^d)$-norm of $\widehat q_t - q_t$ is also small when $r >1$. See Remark~\ref{remark: main W^p,r error small} and Remark~\ref{remark: general W^k,1 error} for further details. Additionally, our method extends to estimating the pointwise difference between $\widehat q_t$ and $q_t$, although we defer this investigation to future work to maintain the manuscript's conciseness. In our proof of Theorem~\ref{theorem: main L^1 theorem}, we leverage the Gagliardo-Nirenberg interpolation inequality with a universal constant, meaning the constant is independent of the dimension. To provide a comprehensive literature review, we include the proof of this dimension-free interpolation inequality as Lemma~\ref{lemma: Gagliardo-Nirenberg}. 

For our convergence analysis of the probability ODE flow at the discretized level, the first step is to reformulate the discrete solution obtained by the $p$-th order Runge-Kutta method as a continuous-time ODE flow using interpolation. We derive an interpolation in \Cref{p:high_order_error}, and crucially, the score function associated with the interpolated ODE flow and the original approximated score function (and their derivatives) are close up to a $p$-th order error, i.e., $\mathcal{O}(h^p)$. Employing the characteristic method described in \cref{section: L^1 error of transport equation} again, the error at the discrete level decomposes into two parts: the score matching error between the generated data distribution $\widehat{q}_t$ and the target distribution $q_t$ along the diffusion process, and the discretization error between the interpolated ODE flow solution and the generated data distribution $\widehat{q}_t$. Consequently, the score matching error and time discretization error do not interact to magnify, thus preserving the time discretization error at the $p$-th order.

Our assumptions on the true data distribution $\muast$ are quite general. In \Cref{assumption:secon-moment}, we assume that $\muast$ has a compact support.  In \Cref{s:prel-estim-q}, we extend our main result to the case where $\muast$ is a Gaussian mixture. We emphasize that under \Cref{assumption:secon-moment}, $\muast$ may not have a density, and the compact support $K_{\ast}$ of $\muast$ can be a submanifold of a much lower dimension in $\mathbb{R}^d$, particularly point masses. Refer to our Example~\ref{example: q_t of delta mass} and Example~\ref{example: q_t of sphere}, where we  observe that assuming $\nabla \log q_t$ is uniformly Lipschitz with a constant independent of $t$ is unreasonable, as it actually tends to $\infty$ as $t \to 0^+$. In Lemma~\ref{Lemma: Hessian estimates} and Lemma~\ref{lemma: Hessian estimates of Gaussian mixture initial data}, we compute and estimate the high-order derivatives of $q_t$ (and $\log q_t$), when $\muast$ has a compact support and when $\muast$ is a Gaussian mixture. For Gaussian mixtures, the error estimates are better than the case when $K_{\ast}$ is a submanifold of a much lower dimension. We also mention in Remark~\ref{remark: other assumptions on initial data} that our methods also apply to other reasonable assumptions on $\muast$ once some simple estimates are satisfied.

\subsection{Notations}

\begin{itemize}
\item Diacritics: $\widehat{\square}$ denotes quantities involve score error, $\widetilde{\square}$ denotes quantities involve time discretization error.
\item Time steps: $0 = t_0 < t_1< \cdots < t_N  = T-\tau$, where $\tau>0$ is a small parameter.  
    \item Distributions on $\mathbb{R}^d$: $q_t, \widehat q_t, \widetilde q_t$ denote forward process, $\varrho_t = q_{T-t}, \widehat \varrho_t = \widehat q_{T-t}, \widetilde \varrho_t = \widetilde q_{T-t}$ reverse process. We also define $\widehat \varepsilon _t(x) \coloneq \widehat q_t(x) - q_t(x), \widetilde \varepsilon _t(x) \coloneq \widetilde q_t(x) - q_t(x)$.
    \item Vector fields from $\mathbb{R}^d$ to $\mathbb{R}^d$:
    Forward process: $\qU_t(x) \coloneq x +  \nabla \log q_t(x)$, $\sU_t \coloneq x + s_{T -t}(x)$,  $\dU_t \coloneq x + \widetilde s_{T -t}(x)$, $\delta_t(x) \coloneq  \sU_t - U_t = s_{T-t}(x) - \nabla\log q_t (x)$, $\widetilde \delta_t(x) \coloneq  \dU_t(x) - U_t(x)$; Reversal process: $V_t \coloneq U_{T-t}$, $\sV_t \coloneq \sU_{T-t}$, $\dV_t \coloneq \dU_{T-t}$; Other vector fields:
    $Z_t(x)$.
    \item $\alpha = (\alpha_1, \alpha_2, \ldots, \alpha_d)$ is a multi-index with nonnegative integers $\alpha_i$'s, $|\alpha| \coloneq \alpha_1 + \alpha_2 + \ldots +\alpha_d$, and we define $\partial_x ^\alpha  \coloneq \partial_{x_1} ^{\alpha_1} \partial_{x_2} ^{\alpha_2} \cdots \partial_{x_d} ^{\alpha_d}$. We also use $\partial_i \coloneq \partial_{x_i}$ for simplicity.
    \item Constants: We use $C_u$ to denote universal constants like $10, 50, 100, 200$, i.e., $C_u$ is independent of the dimension $d$ and other parameters in this paper. %For notation convenience, we will use $A \lesssim B$ to denote $A \leq C_u B$ for two quantities $A,B$. 
    Also, $C_u$ may vary by lines.
    \item Norms: For a vector $x = (x_1, x_2, \ldots, x_d)\in \mathbb{R}^d$, we use $\|x\|=\|x\|_2 \coloneq {(x_1 ^2 + x_2 ^2 + \ldots + x_d ^2)}^{\frac{1}{2}}$, $\|x\|_{\infty} \coloneq \sup_{1 \leq i \leq d} |x_i|$, $\|x\|_{1} \coloneq |x_1| + |x_2| + \ldots+|x_d|$, $\|x\|_{p} \coloneq {({|x_1|}^p + {|x_2|}^p + \ldots+{|x_d|}^p)}^{\frac{1}{p}}$. We similarly define $\|\cdot \|_2$, $\| \cdot \|_{\infty}$, $\| \cdot \|_{1}$, $\| \cdot \|_{p}$,  for matrices or even more general tensors, because we can view them as vectors and forget their tensor structures.
    For a vector-valued function $F(x) = (f_1(x), f_2(x), \ldots, f_m(x)): \mathbb{R}^d \mapsto \mathbb{R}^m$, where $m$ is a positive integer, we usually regard $F(x)$ as a vector in $\mathbb{R}^m$ and similarly use the notations $\|F(x)\|$, $\|F(x)\|_{1}$, $\|F(x)\|_{\infty}$.
\item Function class: We say a vector-valued function $F(x) = (f_1(x), f_2(x), \ldots, f_m(x)): \bR^d\mapsto \bR^m$ as being in $C^r$, if each of its components $f_i(x)$ has continuous first $r$-th derivatives. We say $F(x)$ is in the $L^s$-space $L^{s} (\mathbb{R}^d)$ if for each of its components $f_i(x)$, its $L^s(\mathbb{R}^d)$-norm defined as $\|f_i\|_{L^s(\mathbb{R}^d)} \coloneq {(\int_{\mathbb{R}^d} {|f_i(x)|}^s \de x)}^{\frac{1}{s}}$ is finite. We say $F(x)$ is in the Sobolev space $W^{r,s} (\mathbb{R}^d)$ if for each of its components $f_i(x)$, $\partial_x ^{\alpha} f_i \in L^s(\mathbb{R}^d)$ for each $\alpha$ with $|\alpha| \leq r$. We define the $W^{r,s} (\mathbb{R}^d)$-norm of $f_i$ as $\|f_i\|_{W^{r,s}(\mathbb{R}^d)} \coloneq (\sum_{|\alpha| \leq r} \|\partial_x ^{\alpha} f_i\|_{L^s(\mathbb{R}^d)} ^s )^{\frac{1}{s}}$. 
\end{itemize}

\section{Preliminaries}
\subsection{Score-based Generative Model}
Score-based generative models begin with $d$ dimensional true data samples $\{X_0\}$ following an unknown target distribution $\muast$  with density $q_0$. The objective is to sample new data from the target distribution.
Typically, the score-based generative models usually involve two processes---the forward and reverse processes. 

In the forward process, we start with data samples from $q_0$, and progressively transform the data into noise. This process is often based on the canonical Ornstein-Uhlenbeck~(OU) process given by 
\begin{align}
\label{eq:OU-0}
	\de X_t = -X_t \de t + \sqrt{2} \de B_t, \qquad X_0 \sim q_0, \qquad 0 \leq t \leq T, 
\end{align}
where $(B_t)_{t \in [0, T]}$ is a standard Brownian motion in $\RR^d$. 
The OU process has an analytical solution 
\begin{align}
\label{eq:OU}
	X_t \overset{d}{=} \lambda_t X_0 + \sigma_t W, \qquad W \sim \cN(0, \id_d), 
\end{align}
with  $\lambda_t = e^{-t}$ and $\sigma_t = \sqrt{1 - e^{-2t}}$. The OU process exponentially converges to its stationary distribution, the standard Gaussian distribution $\cN(0, \id_d)$.
Let $q_t$ denote the density of $X_t$, which evolves according to  the following Fokker--Planck equation: 
\begin{align*}
    \del_t q_t= \nabla \cdot \bigl((x +  \nabla \log q_t(x)) q_t\bigr) = \nabla \cdot (\qU_t q_t), \qquad 0 \leq t \leq T,
\end{align*}
with $\qU_t(x) \coloneq x +  \nabla \log q_t(x)$.

By denoting $\varrho_t  = q_{T - t}$, the time reversal process from time $T$ to $0$, satisfies the following partial differential equation (PDE):
\begin{align}\begin{split}
\label{e:reversep}
    \del_t \varrho_t =-\nabla \cdot \bigl((x + \nabla \log q_{T - t}) \varrho_t\bigr) = -\nabla \cdot (V_t \varrho_t), \qquad 0 \leq t \leq T,
\end{split}\end{align}
with $V_t(x) \coloneq x +  \nabla \log q_{T - t}(x)$.
The score function $\nabla\log q_{T- t}(x)$ is typically learned by a neural network trained using score matching techniques with progressively corrupted trajectories $\{X_t\}$ from \eqref{eq:OU-0}. Subsequently,  the reverse PDE can be solved from $\varrho_0 = q_T$ to sample new data from $q_0$.

The reverse  PDE \eqref{e:reversep} is often reformulated into a mean field equation for sampling instead of being directly solved. 
This mean field equation can manifest as stochastic dynamics 
\begin{align}\label{eq:reverse}
	\de Y_t = \left( Y_t + 2 \nabla \log q_{T - t} (Y_t) \right) \de t + \sqrt{2} \de B_t', \qquad Y_0 \sim q_T, \qquad 0 \leq t \leq T,
\end{align}
where $(B_t')_{0 \leq t \leq T}$ is a Brownian motion in $\RR^d$.
This formulation is commonly referred to as the denoising diffusion probabilistic model (DDPM). Alternatively, the mean-field equation can adopt a deterministic dynamics framework in terms of an ordinary differential equation with velocity field $V_t$:
\begin{align}\label{eq:reverse-ode}
	\del_t Y_t =  Y_t +  \nabla \log q_{T - t} (Y_t)  = V_t(Y_t), \qquad Y_0 \sim q_T, \qquad 0 \leq t \leq T,
\end{align}
known as the probability flow ODE. 
Additionally, when $\nabla \log q_{T - t} (x)$, is represented by the learned score function $s_t(x)$, the probability flow ODE \eqref{eq:reverse-ode} becomes 
\begin{align}\label{eq:reverse-ode-score}
	\del_t \sY_t =  \sY_t +  s_t(\sY_t) = \sV_t(\sY_t), \qquad \sY_0 \sim \svarrho_0, 
\end{align}
here the velocity field becomes $\sV_t(x) \coloneq x + s_t(x)$. 
And $\sY_0$ is sampled from $\svarrho_0$, since the density $q_T$ is unknown. $\svarrho_0$ is commonly approximated by the standard Gaussian distribution $\cN(0, \id_d)$, which serves as a reliable approximation of $q_T$ for sufficiently large $T$.
The associated density of $\sY_t$ is denoted as $\svarrho_{t}$, which differs from $\varrho_t$ that describes the density of $Y_t$, due to the score matching error.

\subsection{Time Integrator}
\label{s:time_integrator}
To numerically solve the probability flow ODE \eqref{eq:reverse-ode-score}, a time integrator is essential. Fix a small $\tau>0$, 
we discretize the time interval $[0,\,T-\tau]$ into $N$ time steps $0 = t_0 < t_1< \cdots < t_N  = T-\tau$, typically using a uniform step size $h = (T-\tau)/{N}$, 
Starting from 
an initial condition $\dY_{t_0}$ sampled from $q_T$, 
the time integrator iteratively estimates $\dY_{t_i}$ at time $t_i$.
One commonly used time integrator is the Runge-Kutta method,
the family of explicit $s$-stage $p$-th order Runge-Kutta methods updates $\{\dY_{t_i}\}$ as follows:
\begin{align}
\label{eq:RK-update}
\dY_{t_{i+1}} = \dY_{t_i} + h \sum_{j=1}^s b_j k_j,
\end{align}
where 
\begin{equation}
\label{eq:RK-update-k}
\begin{split}
    &k_1 = \sV_{t_i + c_1 h}\bigl(\dY_{t_i}\bigr),\\
    &k_2 = \sV_{t_i + c_2 h}\bigl(\dY_{t_i} + (a_{21}k_1) h\bigr),\\
    &k_3 = \sV_{t_i + c_3 h}\bigl(\dY_{t_i} + (a_{31}k_1 + a_{32}k_2) h\bigr),\\
    &\qquad\qquad\qquad\vdots
    \\
    &k_s = \sV_{t_i + c_s h}\bigl(\dY_{t_i} + (a_{s1}k_1 + a_{s2}k_2 + \cdots + a_{s,s-1}k_{s-1}) h\bigr).
\end{split}
\end{equation}
The lower triangular matrix $[a_{jk}]$ is called the Runge–Kutta matrix, while the $b_j$ and $c_j$ are known as the weights and the nodes. The stage number $s$ and the parameters are chosen such that 
the local truncation error of \eqref{eq:RK-update} is $\cO(h^{p+1})$. In general, $s \geq p$ and if $p\geq 5$, then $s \geq p+1$.

For example, the forward Euler scheme is the 1-stage first order Runge-Kutta method:
\begin{align*}
\dY_{t_{i+1}} = \dY_{t_i} + h k_1  \qquad k_1 = \sV_{t_i}\bigl(\dY_{t_i}\bigr)
\end{align*}
Heun's method is the 2-stage second order Runge-Kutta method: 
\begin{align*}
\dY_{t_{i+1}} = \dY_{t_i} + \frac{h}{2} (k_1 + k_2)  
\qquad 
k_1 = \sV_{t_i}\bigl(\dY_{t_i}\bigr)
\qquad
k_2 = \sV_{t_{i + 1}}\bigl(\dY_{t_i} + hk_1\bigr)
\end{align*}

\begin{remark}\label{r:RK}
The time discretization error between Runge-Kutta solution $\dY_{t_i}$ and the true solution $\sY_{t_i}$ is typically analyzed through the concept of local truncation error, which is interpreted as follows.
Consider any time interval $[t_i, t_{i+1}]$, solve \eqref{eq:reverse-ode-score} in the time interval with $ \dY_{t_i} = \sY_{t_i}$ analytically, gives 
\begin{align}\label{e:continuous}
 \sY_{t_{i+1}} = \dY_{t_i} + h\sV_{t_i}(\dY_{t_i}) + \frac{h^2}{2}\frac{\de \sV_{t_i}(\dY_{t_i})}{\de t} + \cdots + \frac{h^p}{p!}\frac{\de^{p-1} \sV_{t_i}(\dY_{t_i})}{\de t^{p-1}} +  \frac{1}{(p+1)!}\int_{t_i}^{t_{i+1}}(t_{i+1} - t)^p\frac{\de^{p} \sV_{t}(\dY_t)}{\de t^{p}}\de t.
\end{align}
Similarly for the $s$-stage $p$-th order Runge-Kutta methods \eqref{eq:RK-update}, we can view $\dY_{t_{i} + r}=\widetilde F_r(\dY_{t_i})$ as a function of $r\in[0,h]$ (by replacing $h$ to $r$ in \eqref{eq:RK-update} amd \eqref{eq:RK-update-k}) and $\dY_{t_i}$, and  perform a Taylor expansion around $r=0$
\begin{align}\label{e:discrete}
    \dY_{t_{i+1}}=\dY_{t_i} + h\frac{\de F_0(\dY_{t_i})}{\de r} + \cdots + \frac{h^p}{p!}\frac{\de^{p} F_0(\dY_{t_i})}{\de r^{p}} +  \frac{1}{(p+1)!}\int_{0}^{t_{i+1}-t_i}(h-r)^p\frac{\de^{p+1} \widetilde F_r(\dY_{t_i})}{\de r^{p+1}}\de r.
\end{align}
The  Runge–Kutta matrix $[a_{jk}]$, weights $b_j$ and nodes $c_j$ are carefully chosen such that the coefficients in front of $h, h^2,\cdots, h^p$ (as a function of $\dY_{t_i}$) in \eqref{e:continuous} and \eqref{e:discrete} cancel perfectly.
Thus the Runge-Kutta estimation $\dY_{t_{i+1}}$ satisfies
\begin{align*}
\dY_{t_{i+1}}= \sY_{t_{i+1}}  +  R_{p+1}(\sV_t, \dY_{i},t_i, h), \quad R_{p+1}(\sV_t, \dY_{t_i},t_i, h)=\cO(h^{p+1}). 
\end{align*}
\end{remark}

\section{Main Results}
The probability flow ODE \eqref{eq:reverse-ode} describes the evolution of  $Y_t$; while its counterpart with the estimated score function \eqref{eq:reverse-ode-score} describes the evolution of $\sY_t$.
We denote the density of $Y_t$ and $\sY_t$ as $\varrho_t$ and $\svarrho_t$ respectively.
Then they satisfy the following PDEs
\begin{align}\begin{split}\label{e:defUV}
    &\del_t \varrho_t =-\nabla \cdot (V_t
    \varrho_t),\quad V_t=x+\nabla\log\varrho_t,\\
    &\del_t \widehat\varrho_t = -\nabla \cdot (\sV_t
    \widehat\varrho_t),\quad \sV_t=x+s_t(x).
\end{split}\end{align}
One major focus of our work is to understand the propagation of the score matching error by analyzing the difference between $\svarrho_{t}$ and $\varrho_{t}$.

We make the following assumption on the data distribution $\muast$.
\begin{assumption}\label{assumption:secon-moment}
	The data distribution $\muast$ is positive and compactly supported on a compact set $K_{\ast}$, and we also define $D \coloneq 1+\max_{x \in K_{\ast}} \|x\|_{\infty} $.
\end{assumption}

We assume that the errors incurred during score matching are bounded in an $L^2$-sense.
\begin{assumption}\label{a:score-estimate}
	Fix small $\tau>0$.  There exists a small $\delta>0$, such that the score matching error is bounded by $\delta$ in the sense that
	\begin{align*}
		\int_0^{T-\mft} \E_{\varrho_t}\left[ \|\nabla \log \varrho_{t}(x) - s_t(x)\|_2^2 \right] \rd t\leq \delta^2.
	\end{align*}
\end{assumption}

We assume that score estimates $s_t(x)$ are in $C^2$, and the first two derivatives are bounded by $L_t$, which may depend on time.
\begin{assumption}\label{a:score-derivative}
Fix small $\tau>0$. We assume that the score estimate $s_t(x)$ is $C^2$ with respect to $x$ for any $0 \leq t\leq T-\mft$, and there exists a function $L_t>0$ in $t$, such that 
    \begin{align*}
        \max_{1\leq j \leq d} |s_t ^{(j)}(0)| \leq L_t, \quad 
        \sup_{x \in \mathbb{R}^d} \max_{1 \leq |\alpha| \leq 2} \max_{1\leq j \leq d} | \partial_x ^{\alpha}  s_t ^{(j)}(x)| \leq L_t .
    \end{align*}
Here, we write $s_t(x) = (s_t ^{(1)}(x),s_t ^{(2)}(x),\ldots,s_t ^{(d)}(x))$, and $\alpha = (\alpha_1, \alpha_2, \ldots, \alpha_d)$ is a multi-index with nonnegative integers $\alpha_i$'s, $|\alpha| \coloneq \alpha_1 + \alpha_2 + \ldots +\alpha_d$, and we define $\partial_x ^\alpha  \coloneq \partial_{x_1} ^{\alpha_1} \partial_{x_2} ^{\alpha_2} \cdots \partial_{x_d} ^{\alpha_d}$.
We also define $\mathcal{L} \coloneq \int_{0} ^{T-\mft} L_t \ \de t$.
\end{assumption}
\begin{remark}
These assumptions in \Cref{a:score-derivative} (and similarly in the later \Cref{a:score-high-derivative}) are consistent with empirical evidence suggesting that the learning bias of deep networks favors low-frequency functions, characterized by a small gradient or Hessian \cite{rahaman2019spectral, xu2019frequency}.
Furthermore, the second assumption in  \Cref{a:score-derivative} (similarly, the later \Cref{a:score-high-derivative}) can be replaced with expectation bounds with respect to $\varrho_t$. For example,
    \begin{align*}
        \max_{1 \leq |\alpha| \leq 2} \max_{1\leq i, j \leq d} \int_{\mathbb{R}^d} | \partial_x ^{\alpha}  s_t ^{(j)}(x)| \cdot {{(|x_i|+1)}^{2-|\alpha|}} \varrho_t(x) \de x \leq L_t.
    \end{align*}
    The proofs for Theorem~\ref{theorem: main L^1 theorem} will be similar.

\end{remark}

\begin{theorem}\label{theorem: main L^1 theorem}
 Adopt \Cref{assumption:secon-moment}, \Cref{a:score-estimate} and \Cref{a:score-derivative}, there is a universal constant $C_u >0$, such that the total variation distance between $\varrho_{T-\mft}$ and $\widehat\varrho_{T-\mft}$ is small in the sense that
 \begin{align}\label{e: main L^1 error}
     \TV(\varrho_{T-\mft}, \widehat\varrho_{T-\mft}) \leq C_u \cdot d^{\frac{3}{4}}  \cdot T^{\frac{1}{4}} \cdot {(\mathcal{L}+ T \cdot {\mft}^{-2} \cdot D^3)}^{\frac{1}{2}} \cdot \delta^{\frac{1}{2}} + \TV(\varrho_{0}, \widehat\varrho_{0}).
 \end{align}
If we take the initialization to be the standard normal distribution $\widehat \varrho_0=\cN(0, \id_d)$, then $\TV(\varrho_0, \widehat \varrho_0)\leq C_u e^{-T} \sqrt d D$, which is exponentially small in $T$.
\end{theorem}
\begin{remark}\label{remark: modify error by stronger assumptions}
    We remark that the right hand side of \eqref{e: main L^1 error} is only an upper bound. There are several ways to modify it:
        \begin{itemize}
            \item [(1)] The term $\mft^{-1}$ appears because when $K_{\ast}$ is a submanifold of a dimension much lower than $d$, for example, several points, then $\nabla \log \varrho_{T-\mft} \approx \frac{1}{\sigma_{\mft}^2} \approx \mft ^{-1}$  near $K_{\ast}$ and it becomes much more singular as $\mft \to 0^+$. See Example~\ref{example: q_t of delta mass} and Example~\ref{example: q_t of sphere}. One way to modify this is to modify the Assumption~\ref{a:score-estimate} by a time-weighted score matching error, which will be discussed in Remark~\ref{remark: weighted error}; another way is to assume that our data distribution $\muast$ has a sufficiently regular density, e.g. a Gaussian mixture, then there will be uniform upper bounds (depending on parameters of $\muast$) on $\nabla \log \varrho_{t}$ together with its higher order derivatives, which are independent of the time $t$. So, one can let $\mft \to 0^+$ and the term $\mft ^{-1}$ will not appear in the error estimate. Our proof of Theorem~\ref{theorem: main L^1 theorem} works, essentially verbatim, after using those bounds for $\nabla \log \varrho_{t}$ together with its higher order derivatives.
            See Lemma~\ref{lemma: Hessian estimates of Gaussian mixture initial data} and Remark~\ref{remark: other assumptions on initial data} for more details if $\muast$ does not have a compact support and is possibly a Gaussian mixture.

            \item [(2)] In our numerical simulation, the total variation distance is linear in $\delta$ (See \cref{sec:num}). In \Cref{theorem: L^1 error}, the upper bound in \eqref{e: def L^1 error} is of order $\delta^{1/2}$  because we use \eqref{e: Gagliardo k=2} in Lemma~\ref{lemma: Gagliardo-Nirenberg} to prove Theorem~\ref{theorem: main L^1 theorem} under the Assumption~\ref{a:score-derivative} first. If we can add Assumption~\ref{a:score-high-derivative}, that is, the score estimate $s_t(x)$ has higher derivatives and we can control them, then we can use \eqref{e: Gagliardo general k} in Lemma~\ref{lemma: Gagliardo-Nirenberg} to replace the exponent $\frac{1}{2}$ of $\delta$ in \eqref{e: main L^1 error} with $1- \frac{1}{k}$ for $k \geq 2$, as discussed in Remark~\eqref{remark: higher order Gagliardo}. In this case, the ${(\mathcal{L}+ T \cdot {\mft}^{-2} \cdot D^3)}^{\frac{1}{2}}$ on the right hand side of \eqref{e: main L^1 error} will then be replaced by $C_k {(\mathcal{L}+ T \cdot {\mft}^{-k} \cdot D^{k+1})}^{\frac{1}{k}}$ for some positive constant $C_k$ depending on the $k$, $d^{\frac{3}{4}}$ will be replaced by $d^{\frac{1}{2}+\frac{1}{2k}}$, and $T^{\frac{1}{4}}$ will be replaced by $T^{\frac{1}{2}-\frac{1}{2k}}$.

        %    \item[(3)] One can also replace the term $d$ in \eqref{e: main L^1 error} with the inner dimension $d'<d$ of the distribution $\muast$, if one assumes that $K_{\ast}$ is contained in a $d'$-dimensional hyperplane in $\mathbb{R}^d$. The proofs are the same.
 
        \end{itemize}
\end{remark}

\begin{remark}
    Although in Theorem~\ref{theorem: main L^1 theorem} we only estimate the error up to the time $T-\mft$ instead of the true data $\muast
    =q_0$,  the Wasserstein $2$-distance between $q_0 = \varrho_T$ and $q_{\mft} = \varrho_{T-\mft}$ is actually small if $\mft >0$ is small enough. This is because if we let $\gamma$ be the distribution of $(X_0,X_{\mft})$ on $\mathbb{R}^{2d}$ for $X_t$ defined in \eqref{eq:OU}, then
        \begin{align*}
        W_2(\varrho_T,\varrho_{T-\mft}) ^2 \leq \E_{\gamma} \|X_0 - X_{\mft} \|^2 = \E_{\gamma} \|(1-\lambda_{\mft})X_0 - \sigma_{\mft} W \|^2 \leq 2 {(1-\lambda_{\mft})}^2 \E_{\muast} \|X_0 \| ^2 + 2{\sigma_{\mft} ^2},
        \end{align*}
        where $\lambda_{\mft} = e^{-{\mft}}$ and $\sigma_{\mft} = \sqrt{1- \lambda_{\mft} ^2}$.
    We notice that $1-\lambda_{\mft}\leq \mft$, and by \Cref{assumption:secon-moment}, $\bE_{\mu_*}\|X_0\|^2\leq dD^2$, so 
    \begin{align*}
        W_2(\varrho_T,\varrho_{T-\mft}) ^2\leq 2\tau^2dD^2+4\tau, 
    \end{align*}
    which goes to zero of order $\mft$ as $\mft \to 0^+$.
    
\end{remark}

To numerically solve the probability flow ODE \eqref{eq:reverse-ode-score}, we discretize the time interval $[0,\,T-\tau]$ into $N$ time steps $0 = t_0 < t_1< \cdots < t_N  = T-\tau$, and employ time integration until $t_{N}=T-\tau$ to circumvent the potential singularity at $T$.
Let $\dvarrho_{t_i}$ denote the distribution of $\dY_{t_i}$ obtained by using the Runge-Kutta method described in \Cref{s:time_integrator}. Another focus of our work is to further understand the impact of the time discretization error by analyzing the total variation distance between $\dvarrho_{t}$ and $\varrho_{t}$.

To use the $p$-th order Runge-Kutta method, we assume that 
the estimate score function  $s_t(x)$ grows  linearly, and has bounded first $p+1$ derivatives in $x$ and $t$ in the following assumption.
\begin{assumption}\label{a:score-high-derivative}
    Fix $p\geq 1$ and a small $\tau>0$. Assume there exists a large number $L=L(p,\tau)\geq 1$ the following holds. The approximate score function $s_t(x)$ satisfies $\|s_t(x)\|_2\leq L(\sqrt d+\|x\|_2)$ and is $C^{p+1}$, such that the following holds
    \begin{align*}
      \sup_{x\in \bR^d} \max_{1\leq k+|\alpha|\leq p+1} \max_{1 \leq j \leq d} |\del_t^{k} \partial_{x}^{\alpha} s_t ^{(j)}(x)|   \leq L.
    \end{align*}
    for any $0\leq t\leq T-\mft$.
\end{assumption}

\begin{remark}
    In \Cref{a:score-high-derivative}, besides the upper bounds of the derivatives of $s_t(x)$, we also assumed that  $\|s_t(x)\|_2\leq L(\sqrt d+\|x\|_2)$. This is an easy consequence, if $\|s_t(0)\|_2\leq L\sqrt d$ and $s_t(x)$ is $L$-Lipschitz. 
\end{remark}

  \begin{theorem}
  \label{theorem: main L^1 theorem discretized}
 Adopt \Cref{assumption:secon-moment},\Cref{a:score-estimate}, %\Cref{a:score-derivative} 
 and \Cref{a:score-high-derivative}, there is a universal constant $C_u >0$ and constant $C(p,s, B)>0$ (depending on the stage and the order, and $B:=1+\max_{1\leq i<j\leq s}|a_{ji}|+\max_{1\leq j\leq s}|b_j|$ of the Runge-Kutta method), such that the total variation  between $\varrho_{T-\mft}$ and $\widetilde \varrho_{T-\mft}$ is small in the sense that
 \begin{align}
 \begin{split}\label{e:total_TV}
     \TV(\varrho_{T-\mft}, \widetilde\varrho_{T-\mft})&\leq \underbrace{\TV(\varrho_{0}, \dvarrho_{0})}_{\text{initialization error}}+\underbrace{ C_u \cdot d^{\frac{3}{4}} \cdot  T^{\frac{3}{4}} \cdot {(L+   {\mft}^{-2} \cdot D^3)}^{\frac{1}{2}} \delta^{\frac{1}{2}}}_{\text{score matching error}}\\
     & + \underbrace{C(p,s,B)\cdot d \cdot (LD)^{p+1} \cdot  \log (T/ \mft)\cdot {(hd)}^p}_{\text{discretization error}}.
 \end{split}
 \end{align}
If we take the initialization to be the standard normal distribution $\widetilde \varrho_0=\cN(0, \id_d)$, then $\TV(\varrho_0, \widetilde \varrho_0)\leq C_u  e^{-T} \sqrt d D$ for another positive universal constant $C_u $. So, $\TV(\varrho_0, \widetilde \varrho_0)$ is exponentially small in $T$.
\end{theorem}

\begin{remark}
If the score estimation error is sufficiently small, disregarding logarithmic factors, the probability ODE flow, when solved using the $p$-th order Runge Kutta method, achieves a total variation accuracy of $\varepsilon$ with a total iteration complexity given by:
\begin{align}
    \text{iteration complexity} =\cO\left(\varepsilon^{-1/p}(LDd)^{(p+1)/p}\right).
\end{align}
The term $\varepsilon^{-1/p}$ in the iteration complexity reveals the superiority of the probability ODE flow over the original DDPM sampler, which exhibits a slower iteration complexity proportional to $\varepsilon^{-2}$ 
(\cite{chen2022sampling,chen2023improved,li2023towards}).
\end{remark}

\begin{remark}\label{remark: main W^p,r error small}
    Under the assumptions of Theorem~\ref{theorem: main L^1 theorem discretized}, one can even estimate the $L^1$-norms of higher order derivatives of $\varrho_{T-\mft}-\widehat \varrho_{T-\mft}$ by the Gagliardo-Nirenberg inequality in Lemma~\ref{lemma: Gagliardo-Nirenberg}, as illustrated in Remark~\ref{remark: general W^k,1 error}. By the same arguments in Section~\ref{s:discretize_error}, we can also get error estimates for the $L^1$-norms of higher order derivatives of $\varrho_{T-\mft}-\widetilde \varrho_{T-\mft}$. This means we can obtain the $W^{p,1}$-norms for $\varrho_{T-\mft}-\widetilde \varrho_{T-\mft}$, where $W^{p,1}$ denotes the Sobolev space. By Sobolev inequalities, one can also obtain the corresponding $W^{p-1,r}$-norms for $r >1$, in particular, the $L^r$-norm of $\varrho_{T-\mft}-\widetilde \varrho_{T-\mft}$ with $r >1$.
\end{remark}

\subsection{Proof Outline}
To prove \cref{theorem: main L^1 theorem}, we first consider these two first-order PDEs describing the forward processes
\begin{align}\label{eq:defUhU}
    &\del_t q_t =\nabla \cdot (U_t
    q_t) \qquad \textrm{ and }
    \qquad \del_t \sq_t =\nabla \cdot (\sU_t
    \sq_t).
\end{align}
Here $\sq_t = \svarrho_{T-t}$ and $\sU_t = x + s_{T -t}(x)$, and the second equation describes the density evolution of  $\sY_{T-t} \sim \sq_t$ denoted in \eqref{eq:reverse-ode-score}.
We denote $\delta_t(x) \coloneq  \sU_t - U_t = s_{T-t}(x) - \nabla\log q_t (x)$ as the score matching error, and $\widehat \varepsilon_t(x) \coloneq \sq_t(x) - q_t(x)$ as the error in generated data distribution.
Our goal is to use $\int_{\mft}^T \int_{\mathbb{R}^d} q_t(x) {\|\delta_t(x)\|} ^2  \ \de x \de t$ to bound the $L^1$~error between $\widehat q_t$ and $q_t$ at time $\mft$, i.e.,
\begin{align}
    \TV(q_\mft,\widehat q_\mft)=\int_{\mathbb{R}^d} |\widehat \varepsilon_\mft(x)| \ \de x.
\end{align}

By employing the characteristic method for first-order PDEs \eqref{eq:defUhU}, we can derive a bound for the time derivative of the $L^1$~error as follows:
\begin{align*} 
        \bigg| \frac{\de}{\de t}\int_{\mathbb{R}^d}  | \widehat\varepsilon_{t}(x)  | \ \de x  \bigg| 
        &\leq \int_{\mathbb{R}^d}  \bigg|  \bigl(\nabla \cdot (q_{t}\delta_{t} )\bigr)(x) \bigg| \ \de x.
    \end{align*}
The proof can be found in \cref{section: L^1 error of transport equation}.
Integrating from $\mft > 0$ to $T$, the $L^1$~error is controlled by the gradient of the score error:
\begin{align*}
        \int_{\mathbb{R}^d}  | \widehat\varepsilon_{\mft}(x)  | \ \de x  
        &\leq \int_{\mft}^T \int_{\mathbb{R}^d}  \bigg|  (\nabla \cdot (q_{t}\delta_{t} ))(x) \bigg| \ \de x + \int_{\mathbb{R}^d}  | \widehat\varepsilon_{T}(x)  | \ \de x .
    \end{align*}
Then we use Gagliardo-Nirenberg \cref{lemma: Gagliardo-Nirenberg} and estimations on derivatives of density $q_t$ as presented in \cref{s:prel-estim-q} to control each component of the gradient $\nabla \cdot (q_{t}\delta_{t})$ in the right-hand side in terms of  the $L^2$ score error, leading to
\begin{align*}
         \int_{\mft}^T\int_{\mathbb{R}^d}  \bigg|  (\nabla \cdot (q_{t}\delta_{t} ))(x) \bigg| \ \de x \de t 
         \leq C_u \cdot d^{\frac{3}{4}}  \cdot T^{\frac{1}{4}} \cdot {(\mathcal{L}+ T \cdot {\mft}^{-2} \cdot D^3)}^{\frac{1}{2}} 
 {\bigg(\int_{\mft} ^T \int_{\mathbb{R}^d}q_t(x) {\|\delta_t(x)\|_2 ^2} \ \de x \ \de t \bigg) ^{\frac{1}{4}}  }.
    \end{align*}
    We observe a $1/2$-th order dependence on the $L^2$ score error and linear dependence on the dimensionality $d$. The presence of $\mft^{-2}$ is attributed to the possibility that true data lies on a submanifold of lower dimensionality than $d$.
For detailed proof and improved bounds concerning the Gaussian mixture true data distribution, please refer to \cref{section:score estimation error}.

To prove \cref{theorem: main L^1 theorem discretized}, we first 
interpolate the discrete solution $\{\dY_{t_i}\}_{i=0}^{N}$ obtained by the $p$-th order Runge-Kutta method using interpolation. Then, we obtain a continuous time process on each time interval $[t_i,t_{i+1}]$, which  can then be treated as an ODE flow 
\begin{equation*}
    \partial_t \dY_t = \dV_t(\dY_t),\quad t_i\leq t\leq t_{i+1},
\end{equation*}
where $\widetilde V_t$ is continuous on the $t$-direction when $t \in [t_i,t_{i+1}]$, but it may not be continuous crossing each $t_i$. 
The discrepancy between $\dV_t$ and $\sV_t$ is studied in \cref{p:high_order_error}. Specifically, we analyze
\begin{align}
\label{p:high_order_error-copy}
        \|\dV_{t}(x)-\sV_{t}(x)\|_{\infty},\quad  \|\nabla (\dV_{t}(x)-\sV_{t}(x))\|_{\infty} \leq C(p,s,B) \cdot L\bigl((\sqrt d+\|x\|_2)h\sqrt dL\bigr)^p,
    \end{align}
which essentially represents the Runge-Kutta local truncation error with detailed dimension and Lipschitz constants. Let $\dq_{t}$ denote the density of $\dY_{T-t}$, which satisfies the forward process
\begin{align*}
    \del_t \dq_{t} = 
    \nabla \cdot (\dU_t
    \dq_{t}),\quad \widetilde q_T=\widehat q_T,
\end{align*}
with $\dU_t = \dV_{T-t}$. 
We will quantify the total variation between $\dq_t$ and $q_t$, the density of $Y_{T-t}$.
We define $\widetilde \delta_t(x) \coloneq  \dU_t(x) -  U_t(x)$,  $\widetilde \varepsilon_t(x) \coloneq \dq_t(x) - q_t(x)$. By using the characteristic method described in \cref{section: L^1 error of transport equation} again, the error at the discrete level boils down to the score matching error and time discretization error:
\begin{align*}
        &\phantom{{}={}}\int_{\mathbb{R}^d}  | \widetilde\varepsilon_{\mft}(x)  | \ \de x  - 
        \int_{\mathbb{R}^d}  | \widetilde\varepsilon_{T}(x)  | \ \de x \\
        &\leq 
        \sum_{i=0}^{N-1}\int_{T - t_{i+1}} ^{T - t_{i}} \int_{\mathbb{R}^d}  \bigg|  (\nabla \cdot (q_{t} \delta_{t} ))(x) \bigg| \ \de x \de t + \sum_{i=0}^{N-1}\int_{T - t_{i+1}} ^{T - t_{i}} \int_{\mathbb{R}^d}  \bigg|  (\nabla \cdot (q_{t} (\widetilde \delta_{t}- \delta_{t}) ))(x) \bigg| \ \de x \de t .
    \end{align*}
By using the fact that $\widetilde \delta_{t}- \delta_{t} = \dU_{t}- \sU_{t} = \dV_{T-t}- \sV_{T-t}$, the discretization error becomes
\begin{align*}
\int_{T - t_{i+1}} ^{T - t_{i}} \int_{\mathbb{R}^d}  \bigg|  (\nabla \cdot (q_{t} (\widetilde \delta_{t}- \delta_{t}) ))(x) \bigg| \ \de x  \de t 
=\int_{t_{i}} ^{t_{i+1}} \int_{\mathbb{R}^d}  \bigg|  (\nabla \cdot (q_{T - t} (\dV_{t}- \sV_{t}))(x) \bigg| \ \de x \de t .
\end{align*}
Using the Runge-Kutta local truncation error estimations from \eqref{p:high_order_error-copy}, the discretization error can be bounded as
\begin{align*} 
    \int_{t_{i}} ^{t_{i+1}} \int_{\mathbb{R}^d}  \bigg|  (\nabla \cdot (q_{T-t} (\dV_{t}- \sV_{t}))(x) \bigg| \ \de x \de t    &\leq C(p,s,B)\cdot d  \cdot (LD)^{p+1} \cdot   (dh)^{p}\cdot\int_{t_{i}} ^{t_{i+1}}\frac{1}{\sigma_t^2}\rd t.
    \end{align*}
As a result, the score matching error and time discretization error do not interact to magnify, thereby preserving the time discretization error at $p$-th order, albeit with significant dimensionality dependence.
The detailed proof is in \cref{s:discretize_error}.

\section{Numerical Study}
\label{sec:num}
In this section, we numerically analyze the convergence rate of the probability flow ODE, specifically focusing on a $K$-mode Gaussian mixture target distribution
\begin{equation}
\label{eq:GM}
    q_0(x) = \sum_{k=1}^K w_k \N(x; m_k, C_k).
\end{equation}
The forward process, as denoted in \eqref{eq:OU} with  $\lambda_t = e^{-t}$ and $\sigma_t = \sqrt{1 - e^{-2t}}$, yields
\begin{align*}
	q_t(y) = &  \int_{\RR^d} \frac{1}{(\sqrt{2\pi} \sigma_t)^d} \cdot \exp \left( -\frac{\|y - \lambda_t x\|_2^2}{2\sigma_t^2} \right) q_0(x) \de x 
	=  \sum_{k=1}^K w_k \N(y; \lambda_t m_k, \lambda_t^2 C_k + \sigma_t^2 I).
\end{align*}
The score function takes the following analytical form
\begin{align}
\label{eq:GM-score}
	\nabla_x \log q_t(x)
	= & -\sum_{k=1}^K\frac{w_k\N(x; \lambda_t m_k, \lambda_t^2 C_k + \sigma_t^2 I)}{q_t(x)} (\lambda_t^2C_k + \sigma_t^2I)^{-1}(x - \lambda_t m_k).
\end{align}
In general, the score function is represented by a neural network with inputs $t$ and $x$, trained through score matching with sequentially corrupted training data~\cite{hyvarinen2005estimation,vincent2011connection,song2019generative,song2020score}. However, in our present work, we circumvent the score matching step. Instead, we assume that we have access  to an imperfect score function characterized by the following three types of artificial score errors $\delta(t, x) = s_{T-t}(x) - \nabla \log q_t(x)$ 
\begin{itemize}
    \item constant error : $\delta(t, x) = \delta \frac{1}{\sqrt{d}}$;
    \item linear error: $\delta(t, x) = \delta \frac{x - m}{\sqrt{d}}$;
    \item sinusoidal error: $\delta(t, x) = \delta \sin(x) \frac{x - m}{\sqrt{d}}$.
\end{itemize}
Here $m$ is the mean of the target Gaussian mixture distribution. The $\sin$ function is used for pointwise evaluation, and its product with the following term also represents pointwise multiplication.
In the subsequent numerical investigation, we evaluate the convergence rate of the probability flow ODE for estimating $q_0$ using an analytical score function \eqref{eq:GM-score} with various magnitudes of artificial score errors parameterized by a scalar $\delta$. Specifically, we consider $\delta$ values of $0.005$, $0.01$, $0.02$, $0.04$, $0.08$, and $0.16$.
For the probability flow ODE, we integrate the deterministic reverse process \eqref{eq:reverse-ode-score} using Heun's second-order time integrator until $T=8$. Because the Gaussian mixture target density has no singularities, we integrate the reverse process until the final time $T$ with $\tau = 0$.
Based on our theoretical analysis outlined in \cref{theorem: main L^1 theorem discretized}, to balance the score error and time discretization error, we choose $h^2$ to be approximately equal to $\delta$. Consequently, as we vary the magnitude of the score error $\delta$ from $0.005$ to $0.16$, we adjust the corresponding number of time steps as follows: $N_t = 96$, $64$, $48$, $32$, $24$, and $16$. It is worth noticing the step size must be within the stability regime of the explicit time integrator. 
To initialize the ODE flow, we sample $J=4\times10^4$ particles from the standard Gaussian distribution $\mathcal{N}(x; 0, \id_d)$. All code used to produce the numerical results and figures in this paper are available at
    \url{https://github.com/PKU-CMEGroup/InverseProblems.jl/blob/master/Diffusion/Gaussian-mixture-density.ipynb}\,.

\subsection{One Dimensional Test}
\label{ssec:1D-test}
We first consider a one dimensional 3-mode Gaussian mixture \eqref{eq:GM} with 
$$w = [0.1;\,0.4;\,0.5] \quad m = [-6.0;\,4.0;\,6.0] \quad \textrm{and}\quad C = [0.25;\,0.25;\,0.25].$$
Initially, we explore the convergence behavior of the probability flow ODE in the mean field limit by numerically solving the Fokker-Planck PDE \eqref{e:defUV}. We employ an initial distribution $\mathcal{N}(x; 0, 1)$ and discretize the computational domain $[-10,10]$ into $1000$ cells using a second-order finite volume method \cite{van1974towards}. Integration is performed with Heun's second-order time integrator using a time step of $h = 10^{-3}$ until $T=8$. To ensure accurate understanding at the continuous level, we choose $\Delta x$ and $h$ to be sufficiently small, such that discretization errors are negligible compared to the score error. 
The corresponding score function, reference density $q_t$, and estimated density $\widehat\varrho_{T-t}$ under various imperfect score estimations are illustrated in \cref{fig:1D-pde-density}. While the solution error increases with larger $\delta$, all modes are captured, including the left-side mode with a relatively small density. The convergence rate, assessed in terms of the total variation $\TV(q_0, \widehat{\varrho}_T)$, relative mean error, and relative covariance error, is depicted in \cref{fig:1D-pde-error}. The linear relationship between these errors and $\delta$ is clearly demonstrated.
\begin{figure}[ht]
\centering
    \includegraphics[width=0.9\textwidth]{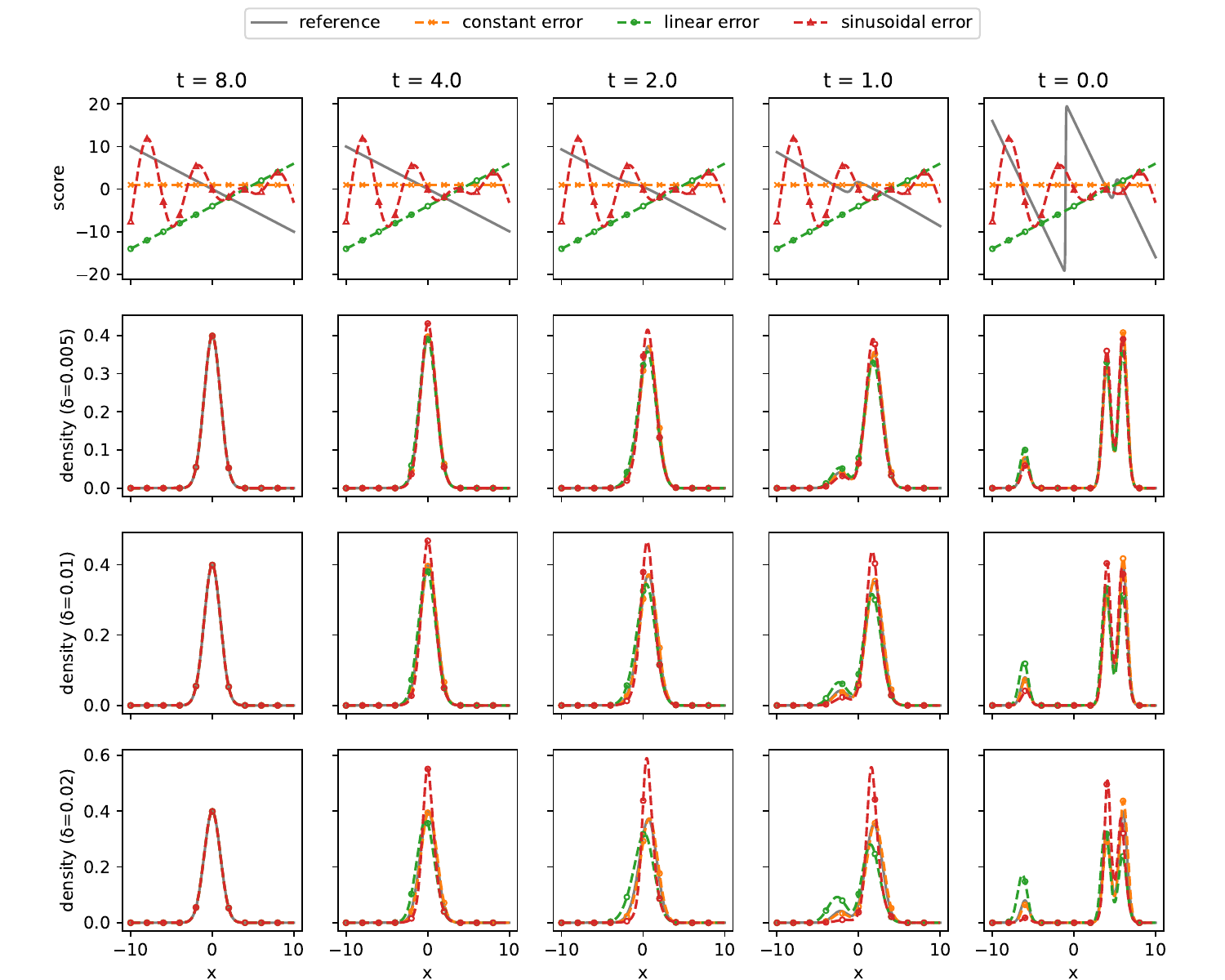}
    \caption{One dimensional test: Density estimations obtained by solving the Fokker-Planck PDE \eqref{e:defUV} numerically with various artificial score errors. From top to bottom: score, estimated density with $\delta = 0.005,\,0.01,\,0.02$. From left to right estimated $q_t$ at $t = 8,\,4,\,2,\,1,\,0$.
     }
    \label{fig:1D-pde-density}
\end{figure}

\begin{figure}[ht]
\centering
    \includegraphics[width=0.9\textwidth]{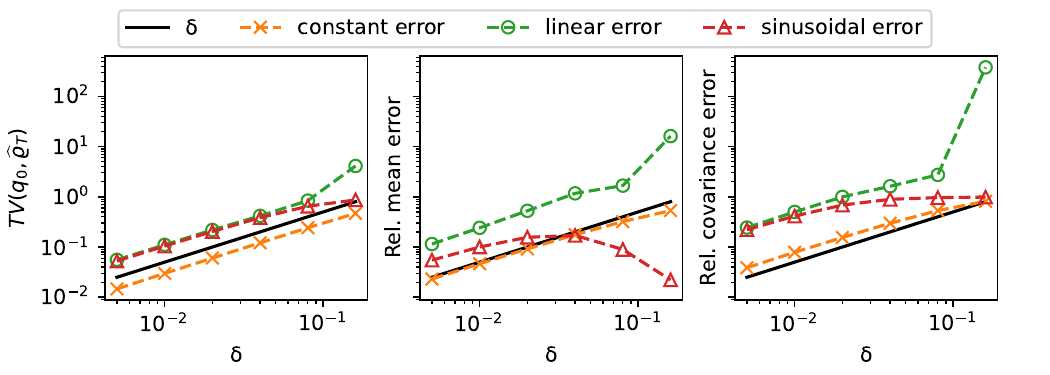}
    \caption{One dimensional test: convergence of the density estimations obtained by solving the Fokker-Planck PDE \eqref{e:defUV} numerically with various artificial score errors.
     }
    \label{fig:1D-pde-error}
\end{figure}

Then, we investigate the convergence of the probability ODE flow by integrating the deterministic reverse process \eqref{eq:reverse-ode-score}. The corresponding score function, reference density $q_t$, and estimated density $\widehat\varrho_{T-t}$ with various imperfect score estimations are depicted in \cref{fig:1D-ode-density}. Kernel density estimates are computed with bandwidth determined by Silverman's rule \cite{silverman2018density} $\Bigl(\frac{4}{J(d+2)}\Bigr)^{1/(d+4)}(1-\frac{t}{2T})$ (interpolating from $0.5$ to $1$). Notably, the estimated densities closely resemble the PDE solution (See \cref{fig:1D-pde-density}), highlighting the significant efficiency of the time integrator with such small numbers of time steps. The convergence rate, evaluated in terms of the total variation $\TV(q_0, \widehat{\varrho}_T)$, relative mean error, and relative covariance error, is depicted in \cref{fig:1D-ode-error}. The linear relationship between these errors and $\delta$ (or $h^2$) is clearly demonstrated.

\begin{figure}[ht]
\centering
    \includegraphics[width=0.9\textwidth]{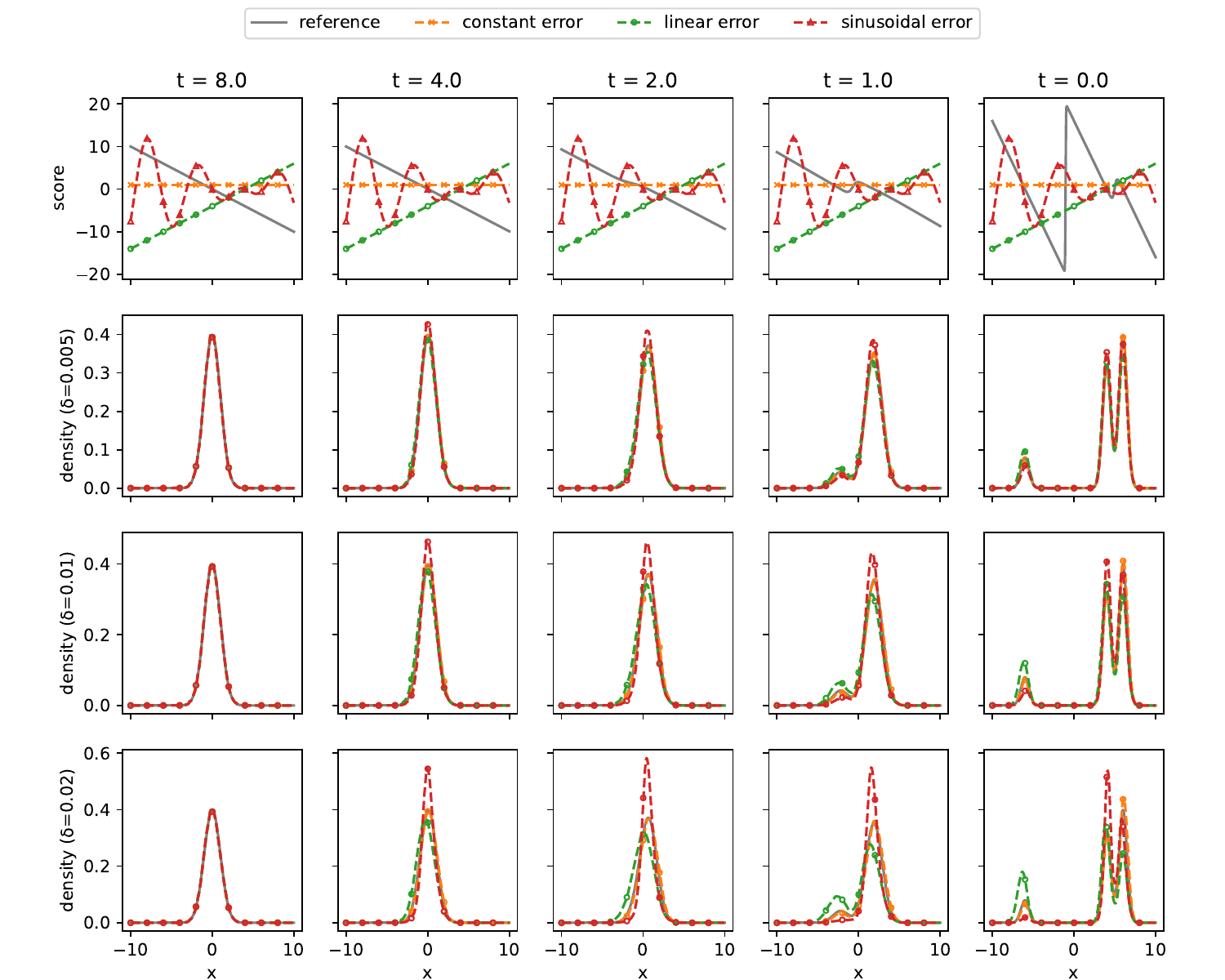}
    \caption{One dimensional test: density estimations obtained by solving the probability flow ODE with Heun's method with various artificial score errors. From top to bottom: score, estimated density with $\delta = 0.005,\,0.01,\,0.02$. From left to right estimated $q_t$ at $t = 8,\,4,\,2,\,1,\,0$.
     }
    \label{fig:1D-ode-density}
\end{figure}

\begin{figure}[ht]
\centering
    \includegraphics[width=0.9\textwidth]{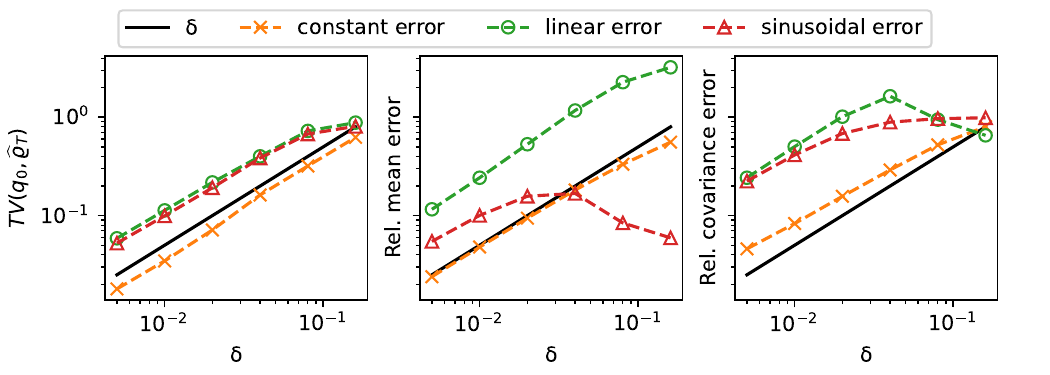}
    \caption{One dimensional test: convergence of the density estimations obtained by solving the probability flow ODE with Heun's method, where $h^2  \approx \delta $ with various artificial score errors.
     }
    \label{fig:1D-ode-error}
\end{figure}

\subsection{High Dimensional Test}
\label{ssec:HD-test}
Finally, we consider $d$ dimensional 5-mode Gaussian mixtures \eqref{eq:GM}. The weights are sampled uniformly $w_k \sim $ Uniform$[0,1]$ and are then normalized to sum to 1. The means are generated from a Gaussian distribution, $m_k \sim \N(0, 3^2\id_d)$, and the covariance matrices $C_k$ are generated as $C_k = \frac{1}{8}(W_k^T W_k/d + \id_d)$ with $(W_k)_{ij} \sim \N(0,1)$ for $i,j = 1,\,\cdots,\, d$. 

We investigate the convergence of the ODE flow by integrating the deterministic reverse process \eqref{eq:reverse-ode-score} with the same setup as before. For $d=128$, we visualize the results in terms of the marginal density for the first dimension, including its score, reference density $q^1_{t}$, and estimated density $\widehat\varrho^1_{T-t}$ with various artificial score errors, as depicted in \cref{fig:high-D-ode-density}. We compute kernel density estimates with the bandwidth determined by the same Silverman's rule \cite{silverman2018density} as in \cref{ssec:1D-test}.
Surprisingly, the estimated densities are as good as those of the one-dimensional test (See \cref{fig:1D-ode-density}). We further consider $d = 8$ and $32$ by generating the target density $q_0$ through marginalizing the 128-dimensional  Gaussian mixture target density.
The convergence rate, evaluated in terms of the total variation  of the marginal density $\TV(q^1_{0}, \widehat{\varrho}^1_{T})$, relative mean error, and relative covariance error, is depicted in \cref{fig:high-D-ode-error}. The linear relationship between these errors and $\delta$ (or $h^2$) is clearly demonstrated. Additionally, we also explore the scenario without score matching errors for comprehensive analysis. The convergence rate, evaluated using the same error indicators, is presented in \cref{fig:high-D-ode-discretization-error}, showing the quadratic relationship with $h^2$.
Notably, in this high dimensional test study, we do not observe any dimension dependence.

\begin{figure}[ht]
\centering
    \includegraphics[width=0.9\textwidth]{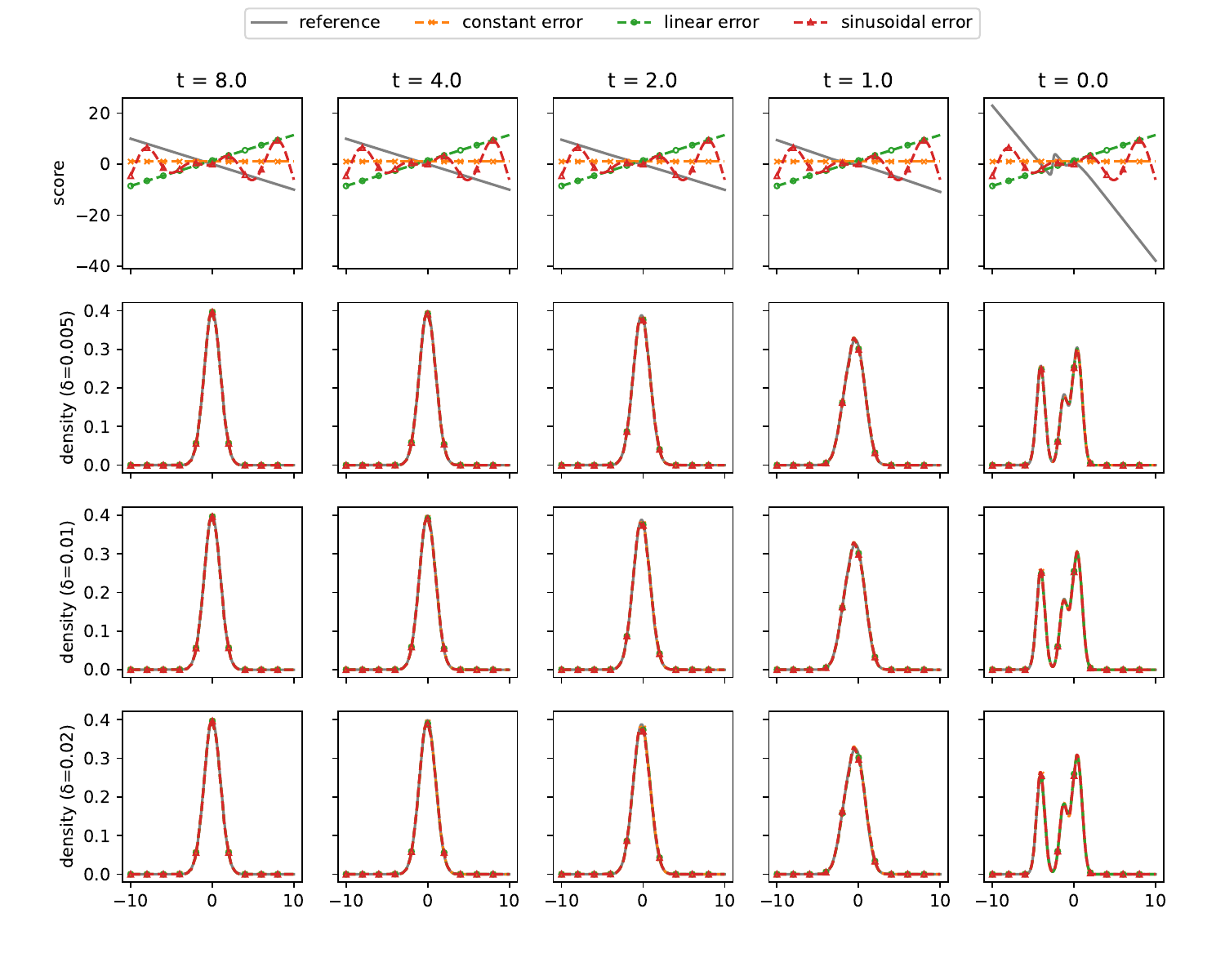}
    \caption{128 dimensional test: 
    marginal densities obtained by solving the probability flow ODE with Heun's method with various artificial score errors. From top to bottom: score, estimated densities with $\delta = 0.005,\,0.01,\,0.02$. From left to right estimated $q_t$ at $t = 8,\,4,\,2,\,1,\,0$.
     }
    \label{fig:high-D-ode-density}
\end{figure}

\begin{figure}[ht]
\centering
    \includegraphics[width=0.9\textwidth]{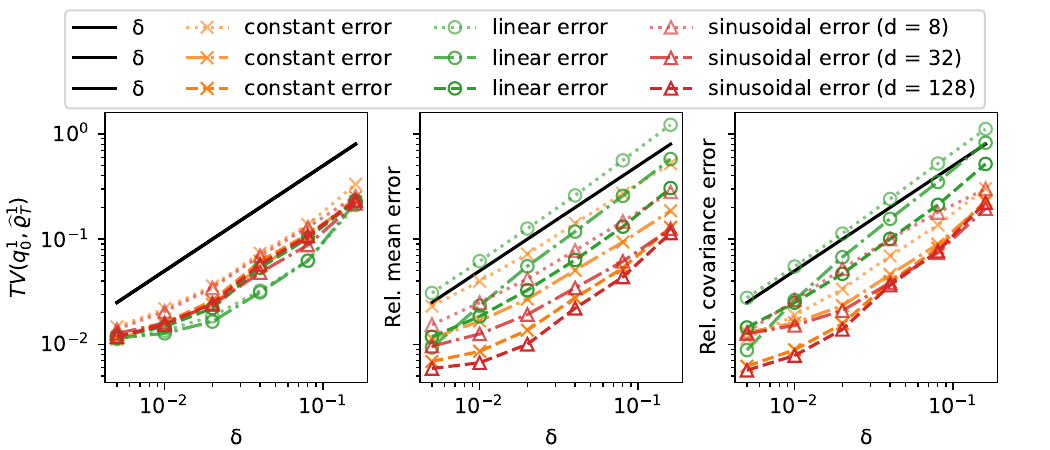}
    \caption{High dimensional test: convergence of the density estimations obtained by solving the probability flow ODE with Heun's method, where $h^2  \approx \delta $ with various artificial score errors. The dotted lines, dash dot lines  and dashed lines  indicate $d=8, 32$ and, $128$ dimension tests, respectively.
     }
    \label{fig:high-D-ode-error}
\end{figure}

\begin{figure}[ht]
\centering
    \includegraphics[width=0.9\textwidth]{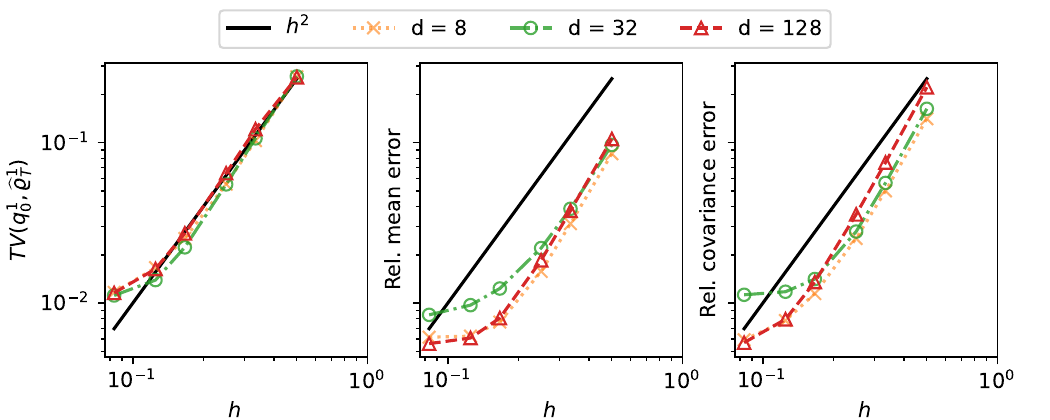}
    \caption{High dimensional test: convergence of the density estimations with no score errors obtained by solving the probability flow ODE with Heun's method, across different time step sizes $h$. The dotted cross lines, dash dot circle lines  and dashed triangle lines indicate $d=8, 32$ and, $128$ dimension tests, respectively.
     }
    \label{fig:high-D-ode-discretization-error}
\end{figure}
\newpage

%we can rewrite \eqref{e:dtKL2}
%\begin{align}
% \del_t KL(\varrho_t|\widehat \varrho_t)
% &=    -\int (U_t-V_t)\varrho_t \nabla \log \frac{ \varrho_t}{\widehat \varrho_t}\rd x\\
% &=\int (\nabla \log \varrho_t -s_t)\varrho_t \nabla \log \frac{ \varrho_t}{\widehat \varrho_t}\rd x\\
% &=\int (\nabla \log \varrho_t -s_t)^2\varrho_t \rd x
% -\int(\nabla \log \varrho_t -s_t)(\nabla \log \widehat \varrho_t -s_t)\varrho_t
%\end{align}
%If the error from score matching is small, i.e. $\nabla \log \varrho_t-s_t$ is small, the the first term can be bounded. For the second term, to bound it we need to upper bound $|(\nabla \log \widehat \varrho_t -s_t)|$. This probably can be done if we assume some regularity for the score function $s_t(x)$, then we can probably use the following relation 
%\begin{align}
%    \del_t \log \widehat \varrho_t =\nabla V_t +V_t \nabla \log \widehat\varrho_t
%\end{align}
%to obtain certain bound for $|(\nabla \log \widehat \varrho_t -s_t)|$. Then this can give some bound for the ODE setting. In the paper by Yuxin, they require the extra assumption that on the Jacobian of $\nabla \log\varrho_t-s_t$:
%$\|J(\nabla \log\varrho_t-s_t)\|$.

\section{Conclusion}
In this study, we have investigated the convergence of score-based generative model based on the probability flow ODE, both theoretically and numerically.
Our analysis provided theoretical convergence guarantees at both continuous and discrete levels. 
Additionally, our numerical studies, conducted on problems up to 128 dimensions, provided empirical verification of our theoretical findings. One notable observation is the superior error bound of $\mathcal{O}(\delta + h^p)$, indicating potentially sharper estimations with improved  dimension and score matching error dependence. Moreover, conducting rigorous numerical-based analyses with neural network-based score estimation errors is also a focus for future research.

\vspace{0.5in}
\noindent {\bf Acknowledgments} 
DZH is supported by the high-performance computing platform of Peking University.
JH is supported by NSF grant DMS-2331096, and the Sloan research fellowship. 
ZJL would like to thank Professor Guido De Philippis for the insightful discussions.
\bibliography{References.bib}
\bibliographystyle{abbrv}

\appendix

\section{Total Variation Estimates along Probability Flow}
\label{section: L^1 error of transport equation}

In this section, we prove the following general theorem for continuity equations, which is independent of the particular choice of $U_t, \sU_t$ (equivalently, $V_{T-t}, \sV_{T-t}$) in \eqref{e:defUV}.
\begin{theorem}\label{theorem: L^1 error}
    Fix any $0 < \mft < T$. Let $\widehat q_t(x), q_t(x) \in C^1([\mft, T] \times \mathbb{R}^d) \cap L^1 ([\mft, T] \times \mathbb{R}^d) $ solve the following two continuity equations on $\mathbb{R}^d$ respectively,
    \begin{align*}
    \del_t q_t =\nabla \cdot (U_t
    q_t),\quad
    \del_t \widehat q_t =\nabla \cdot (\sU_t
    \widehat q_t).
    \end{align*}
    We also assume that for $t \in [\mft, T]$, $U_t, \sU_t$ are locally Lipschitz on $\mathbb{R}^d$. Then, if we denote
    $\delta_t(x) \coloneq  \sU_t(x) - U_t(x)$, $\widehat\varepsilon_t(x) \coloneq \widehat q_t(x) - q_t(x)$, we have that, 
    \begin{align*}
        \bigg| \int_{\mathbb{R}^d}  | \widehat\varepsilon_{\mft}(x)  | \ \de x  - \int_{\mathbb{R}^d}  | \widehat\varepsilon_{T}(x)  | \ \de x \bigg | \leq \int_{\mft} ^T E(t) \ \de t, \quad \text{with } E(t) \coloneq \int_{\mathbb{R}^d}  \bigg|  (\nabla \cdot (q_{t}\delta_{t} ))(x) \bigg| \ \de x.
    \end{align*}
\end{theorem}
\begin{proof}[First Proof]

Fix $\mft >0$. For $t >0$, we define the total variation between $\widehat q_t(x)$ and $q_t(x)$ as
    \begin{align}\label{e: def L^1 error}
        f(t) \coloneq \int_{\mathbb{R}^d}  | \widehat\varepsilon_{t+\mft}(x)  | \ \de x.
    \end{align}
We use $x_t \coloneq T_t(x)$ to denote the solution of the ODE (also called the characteristic line)
\begin{align*}
    \frac{\de}{\de t} T_t(x) = - \sU_{\mft+t}(T_t(x)),
\end{align*}
with the initial data $x_0 = T_0(x) = x$.
We notice that, by change of variables (we can do this because $\sU_t$ is locally Lipschitz on $\mathbb{R}^d$, so, $T_t$ is a diffeomorphism of $\mathbb{R}^d$),
    \begin{align}\label{eqn:f(t) change of variable}
        f(t) = \int_{\mathbb{R}^d}  | \widehat\varepsilon_{t+\mft}(x_t)   | JT_t(x)\ \de x,
    \end{align}
where $JT_t(x) = |\det(\nabla T_t(x))|$. A direct computation shows that
\begin{align*}
    \frac{\de}{\de t} \widehat\varepsilon_{t+\mft}(x) = (\nabla \cdot (\sU_{t+\mft} \widehat\varepsilon_{t+\mft}))(x)  + (\nabla \cdot (q_{t+\mft}\delta_{t+\mft} ))(x).
\end{align*}
Hence,
\begin{align*}
    \begin{split}
        \frac{\de}{\de t} \left[\widehat\varepsilon_{t+\mft}(x_t) \right]&= \left(\frac{\de}{\de t} \widehat\varepsilon_{t+\mft}\right)(x_t) + \left[\nabla \widehat\varepsilon_{t+\mft} (x_t) \cdot \frac{\de}{\de t}x_t \right]
        \\  &= (\nabla \cdot (\sU_{t+\mft} \widehat\varepsilon_{t+\mft}))(x_t)  + (\nabla \cdot (q_{t+\mft}\delta_{t+\mft} ))(x_t) + \left[\nabla \widehat\varepsilon_{t+\mft} (x_t) \cdot \left( - \sU_{\mft+t}(x_t)\right)\right]
        \\  &= (\nabla \cdot \sU_{t+\mft})(x_t) \cdot \widehat\varepsilon_{t+\mft}(x_t) + (\nabla \cdot (q_{t+\mft}\delta_{t+\mft} ))(x_t).
    \end{split}
\end{align*}
Jacobi's formula gives that
\begin{align*}
    \frac{\de}{\de t} \det(\nabla T_t(x)) = \trace \bigg(\bigg(\frac{\de}{\de t}\nabla T_t(x) \bigg) \cdot {\bigg(\nabla T_t(x)\bigg)}^{-1} \bigg) \cdot \det(\nabla T_t(x)).
\end{align*}
We also notice that 
\begin{align*}
    \frac{\de}{\de t} \nabla T_t(x) = \nabla \frac{\de}{\de t} T_t(x) =- \nabla (\sU_{\mft+t}(T_t(x))) = - (\nabla \sU_{\mft+t}) (T_t(x)) \cdot \nabla T_t(x).
\end{align*}
So,
\begin{align*}
    \frac{\de}{\de t} \det(\nabla T_t(x)) = -\trace \bigg((\nabla \sU_{\mft+t}) (T_t(x))  \bigg) \cdot \det(\nabla T_t(x)) = -(\nabla \cdot \sU_{\mft+t}) (T_t(x)) \cdot \det(\nabla T_t(x)).
\end{align*}
Also, by \eqref{eqn:f(t) change of variable}, we have that
    \begin{equation*}
        \begin{split}
            \bigg| \frac{\de}{\de t} f(t) \bigg| &= \bigg| \int_{\mathbb{R}^d}  \frac{\de}{\de t} \big| \widehat\varepsilon_{t+\mft}(x_t)   \det(\nabla T_t(x)) \big|\ \de x \bigg| \leq \int_{\mathbb{R}^d} \bigg|  \frac{\de}{\de t} \big| \widehat\varepsilon_{t+\mft}(x_t)   \det(\nabla T_t(x)) \big|  \bigg| \ \de x
            \\  &= \int_{\mathbb{R}^d} \bigg|  \frac{\de}{\de t} \big[ \widehat\varepsilon_{t+\mft}(x_t)   \det(\nabla T_t(x)) \big] \bigg| \ \de x .
        \end{split}
    \end{equation*}
Hence,
    \begin{equation*}
    \begin{split}
        \bigg| \frac{\de}{\de t} f(t) \bigg| &\leq \int_{\mathbb{R}^d}  \bigg| \big[ (\nabla \cdot \sU_{t+\mft})(x_t) \cdot \widehat\varepsilon_{t+\mft}(x_t) + (\nabla \cdot (q_{t+\mft}\delta_{t+\mft} ))(x_t) \big]\cdot \det(\nabla T_t(x)) 
        \\ &\quad +  \widehat\varepsilon_{t+\mft}(x_t)    \cdot ( -(\nabla \cdot \sU_{\mft+t}) (T_t(x)))\cdot \det(\nabla T_t(x)) \bigg| \ \de x
        \\ &= \int_{\mathbb{R}^d}  \bigg|  (\nabla \cdot (q_{t+\mft}\delta_{t+\mft} ))(x_t) \bigg| \cdot JT_t(x) \ \de x
        \\ &= \int_{\mathbb{R}^d}  \bigg|  (\nabla \cdot (q_{t+\mft}\delta_{t+\mft} ))(x) \bigg| \ \de x = E(t+\mft),
            \end{split}
    \end{equation*}
where the second last equality is by change of variables again. 
Hence, for any $t'>0$,
    \begin{align}\label{e: L^ error estimate}
        f(0) \leq \int_0 ^{t'} E(t+\mft) \ \de t + f(t').
    \end{align}
One can also use the characteristic line method starting from $t = T$ to $t = \mft$, and then obtain an inverse inequality. Hence, we finish the proof of the theorem.
\end{proof}

\begin{remark}
    Notice that in this first proof, we exchange the order of integrals and derivatives. We can do this in our problem setting \eqref{e:defUV}, because by Assumption~\ref{assumption:secon-moment}, $\muast$ is compactly supported, so $q_t(x)$ always has an exponential tail as $\|x\| \to +\infty$. $\widehat q_{t} (x)$ also has such a property by our Lemma~\ref{lemma: hat qt exponential tail}.
\end{remark}

Indeed, one can prove the following more general theorem with reasonable assumptions, essentially verbatim, by the same method.
\begin{theorem}\label{theorem: general L^1 error}
    Fix any $0 < \mft < T$. Let $p_t(x) \in C^1([\mft, T] \times \mathbb{R}^d) \cap L^1 ([\mft, T] \times \mathbb{R}^d) $ solve the following continuity equation on $\mathbb{R}^d$ with $h_t(x) \in L^1 ([\mft, T] \times \mathbb{R}^d)$,
    \begin{align*}
    \del_t p_t(x) =(\nabla \cdot (Z_t
    p_t))(x) + h_t(x).
    \end{align*}
    We also assume that for $t \in [\mft, T]$, $Z_t$ is locally Lipschitz on $\mathbb{R}^d$. Then, for almost all $t \in [\mft,T]$, we have that, 
    \begin{align*}
        \bigg| \frac{\de}{\de t} \int_{\mathbb{R}^d}  | p_{t}(x)  | \ \de x  \bigg | \leq \int_{\mathbb{R}^d}  |h_t(x)| \ \de x.
    \end{align*}
Hence,
    \begin{align*}
        \bigg| \int_{\mathbb{R}^d}  | p_{\mft}(x)  | \de x - \int_{\mathbb{R}^d}  | p_{T}(x)  | \de x \bigg | \leq \int_{\mft} ^T \int_{\mathbb{R}^d}  |h_t(x)| \ \de x.
    \end{align*}
    
\end{theorem}
\begin{remark}\label{remark: general W^k,1 error}
    We will prove a Gagliardo-Nirenberg interpolation inequality with dimension free constants in Lemma~\ref{lemma: Gagliardo-Nirenberg}. With this inequality, if one can know that $\int_{\mathbb{R}^d}  |\partial^2 _{ii} p_{\mft}(x)  | \de x$ is bounded for some $i \in \llbracket 1, d \rrbracket$, then one can use Lemma~\ref{lemma: Gagliardo-Nirenberg} to see that $\int_{\mathbb{R}^d}  |\partial_i p_{\mft}(x)  | \de x$ is also small if $\int_{\mathbb{R}^d}  | p_{\mft}(x)  | \de x$ is small. Similar arguments work for higher order derivatives of $p_{\mft}(x)$. Moreover, if we can conclude that the $W^{k,1}(\mathbb{R}^d)$-norms of $p_{\mft}(x)$ is small, where $W^{k,1}$ denotes the Sobolev space, we can use Sobolev inequalities to conclude that the $W^{k-1,r}(\mathbb{R}^d)$-norms of $p_{\mft}(x)$ for a corresponding $r>1$ is also mall. In particular, we can conclude that the $L^r(\mathbb{R}^d)$-norm of $p_{\mft}(x)$ is small.
\end{remark}

Before we give the second proof of Theorem~\ref{theorem: L^1 error} and Theorem~\ref{theorem: general L^1 error}, we need to explain our intuition a little bit.
With those notations in Theore~\ref{theorem: general L^1 error},
    we notice that for
    \begin{align*}
        g(t) \coloneq \int_{\mathbb{R}^d}  |p_{t}(x)  | \ \de x,
    \end{align*}
we have that 
    \begin{align*}
        \frac{\de}{\de t} g(t) = \int_{\mathbb{R}^d}  \sign(p_{t}(x)) \frac{\de}{\de t}p_{t}(x)   \ \de x = \int_{ \{p_{t} \geq 0 \} }   \frac{\de}{\de t}p_{t}(x)   \ \de x - \int_{ \{p_{t} < 0 \} }   \frac{\de}{\de t}p_{t}(x)   \ \de x .
    \end{align*}
If the boundary $\partial \{p_{t} \geq 0 \} $ consists of $(d-1)$-dimensional piecewise smooth submanifolds (or at least rectifiable sets), then by the divergence theorem, we know that 
\begin{align*}
    \int_{ \{p_{t} \geq 0 \} }   (\nabla \cdot (Z_{t} p_{t}))(x)   \ \de x = \int_{\partial \{p_{t} \geq 0 \}  } \textbf{n}(x) \cdot (Z_{t} p_{t})(x) \ d\mathcal{H}^{d-1}(x) = 0,
\end{align*}
where $\textbf{n}(x)$ is the unit outer normal vector and $\mathcal{H}^{d-1}(\cdot)$ is the $(d-1)$-dimensional surface measure. Similarly,
\begin{align*}
    \int_{ \{p_{t} < 0 \} }   (\nabla \cdot (Z_{t} p_{t}))(x)   \ \de x  = 0.
\end{align*}
So, 
\begin{align*}
        \frac{\de}{\de t} g(t) = \int_{\mathbb{R}^d}  \sign(p_{t}(x))  h_t(x)  \ \de x, \  \text{and } \bigg| \frac{\de}{\de t} g(t) \bigg| \leq \int_{\mathbb{R}^d}  |h_t(x)| \ \de x.
\end{align*}
However, in general, we cannot know whether the boundary set $\partial \{p_{t} \geq 0 \} $ is always of $(d-1)$-dimension. For an arbitrarily given smooth function, its zero level set can also be arbitrarily strange and does not necessarily need to be of $(d-1)$-dimension.

The following second proof of Theorem~\ref{theorem: general L^1 error} is inspired by~\cite{Ambrosio2004Transport} and the communication with Professor Guido De Philippis. We also assume that $(\|Z_t(x)\| \cdot |p_t(x)|) \in L^1 (\mathbb{R}^d \times [\mft, T])$.
\begin{proof}[Second Proof]
    Let $a >0$. We let $\beta_{a}(s) = \sqrt{a^2 + s^2}$ be a function on $\mathbb{R}$ which approximates the function $|s|$ as $a \to 0^+$. Then, the function $p_{a,t}(x) \coloneq \beta_{a}(p_t(x))$ on $\mathbb{R}^d$ solves the equation
        \begin{align*}
            \begin{split}
                \del_t p_{a,t}(x) &= ((\nabla \cdot Z_t)(x)) p_t(x) \beta_a ' (p_t(x)) + Z_t(x) \cdot \nabla p_{a,t}(x) + h_t(x) \beta_a ' (p_t(x))
                \\  &= \nabla \cdot (Z_t(x)  p_{a,t}(x)) + ((\nabla \cdot Z_t)(x)) [p_t(x) \beta_a ' (p_t(x))- p_{a,t}(x)]+ h_t(x) \beta_a ' (p_t(x)).
            \end{split}
        \end{align*}
For any $R>0$, consider the integral of the above equation on the ball $B_R \coloneq B_R(0) \subset \mathbb{R}^d$, we have that 
    \begin{align*}
        \begin{split}
            \frac{\de}{\de t} \int_{B_R} p_{a,t}(x) \de x &= \int_{\partial B_R} (Z_t(x)  p_{a,t}(x)) \cdot \frac{x}{\|x\|_2} d\mathcal{H}^{d-1}(x) + \int_{B_R} ((\nabla \cdot Z_t)(x)) [p_t(x) \beta_a ' (p_t(x))- p_{a,t}(x)] \ \de x
            \\ & \quad + \int_{B_R} h_t(x) \beta_a ' (p_t(x)) \ \de x,
        \end{split}
    \end{align*}
and then for any $t',t'' \in [\mft,T]$ and $t' > t''$,
    \begin{align*}
        \begin{split}
            \bigg| \int_{B_R} p_{a,t''}(x) \de x - \int_{B_R} p_{a,t'}(x) \de x \bigg|&= \bigg| \int_{t''} ^{t'} \int_{\partial B_R} (Z_t(x)  p_{a,t}(x)) \cdot \frac{x}{\|x\|_2} d\mathcal{H}^{d-1}(x) \de t 
            \\ & \quad + \int_{t''} ^{t'} \int_{B_R} ((\nabla \cdot Z_t)(x)) [p_t(x) \beta_a ' (p_t(x))- p_{a,t}(x)] \ \de x \de t
            \\ & \quad + \int_{t''} ^{t'} \int_{B_R} h_t(x) \beta_a ' (p_t(x)) \ \de x \de t \bigg|.
        \end{split}
    \end{align*}
Notice that $|\beta_a ' (s)| = \frac{|s|}{\sqrt{a^2 + s^2}} \leq 1$, $|p_t(x) \beta_a ' (p_t(x))- p_{a,t}(x)| = \frac{a^2}{\sqrt{a^2 + p_t(x) ^2}} \leq 1$, and for any given $x \in B_R$, $\lim_{a \to 0^+} \frac{a^2}{\sqrt{a^2 + p_t(x) ^2}} = 0$. By the dominated convergence theorem, we can let $a \to 0^+$ on both sides and obtain that
\begin{align*}
        \begin{split}
            \bigg| \int_{B_R} |p_{t''}(x)| \de x - \int_{B_R} |p_{t'}(x)| \de x \bigg|&= \bigg| \int_{t''} ^{t'} \int_{\partial B_R} (Z_t(x)  |p_{t}(x)|) \cdot \frac{x}{\|x\|_2} d\mathcal{H}^{d-1}(x) \de t 
            \\ & \quad + \int_{t''} ^{t'} \int_{B_R} ((\nabla \cdot Z_t)(x)) \cdot 0 \ \de x \de t
            \\ & \quad + \int_{t''} ^{t'} \int_{B_R} h_t(x) \sign(p_t(x))\ \de x \de t \bigg|
        \\  &\leq \int_{t''} ^{t'} \int_{\partial B_R} \|Z_t(x)  \|_2 \cdot |p_{t}(x)|  d\mathcal{H}^{d-1}(x) \de t 
        \\  & \quad + \int_{t''} ^{t'} \int_{B_R} |h_t(x)|\ \de x \de t .
        \end{split}
    \end{align*}
Because we assumed that $(\|Z_t(x)\|_2 \cdot |p_t(x)|) \in L^1 (\mathbb{R}^d \times [\mft, T])$, we can choose a sequence  $\{R_j\}$ with $R_j \to +\infty$, such that the first term on the right hand side goes to $0$. So, by the monotone convergence theorem, after passing $R_j \to +\infty$, we obtain that for any $t',t'' \in [\mft,T]$ and $t' > t''$, 
    \begin{align*}
        \bigg| \int_{\mathbb{R}^d} |p_{t''}(x)| \de x - \int_{\mathbb{R}^d} |p_{t'}(x)| \de x \bigg| \leq \int_{t''} ^{t'} \int_{\mathbb{R}^d} |h_t(x)|\ \de x \de t .
    \end{align*}
\end{proof}

\section{Preliminary Estimates on Forward and Backward Density}
\label{s:prel-estim-q}

In this section, $U_t, \sU_t$ are defined as in \eqref{e:defUV}.
Under \Cref{assumption:secon-moment}, we take $\lambda_t = e^{-t}$ and $\sigma_t = \sqrt{1- \lambda_t ^2}$ ($t>0$), and we also assume that
    \begin{align}\label{e: definition of q_t}
        q_t (y) \coloneq \int_{\mathbb{R}^d} \frac{1}{{(\sqrt{2\pi} \sigma_t)}^{d}} \cdot e^{-\frac{\|y-x\|_2^2}{2{\sigma_t}^2}} \cdot \muast\bigg(\frac{x}{\lambda_t}\bigg) \frac{1}{{\lambda_t}^d} \ \de x,
    \end{align}
which satisfies the first-order PDE: $\del_t q_t =\nabla \cdot (U_t q_t)$ for $U_t=x+\nabla\log q_t (x)$.
We remark that although we use the notation $\muast(x / \lambda_t ^d) / \lambda_t ^d$ in the definition \eqref{e: definition of q_t} of $q_t$, because the data distribution $\muast$ can be supported on a submanifold $K_{\ast}$, or more general lower dimensional rectifiable sets, the meaning of it is actually the push-forward measure defined by $(\lambda_t)_{\#} \muast(A) \coloneq \muast(\lambda_t ^{-1} A)$ for any measurable set $A \subset \mathbb{R}^d$. So, the rescaling factor $\lambda_t ^d$ is actually $\lambda_t ^k$ for $k \leq d$ when $K_{\ast}$ is a $k-$dimensional submanifold in $\mathbb{R}^d$. But this notation doesn't affect our computations. As readers will see, the only property we will use is that in the integrand of \eqref{e: definition of q_t}, $x / \lambda_t \in K_{\ast}$ and hence $\|x\|_{\infty} \leq \lambda_t D$.

\begin{example}\label{example: q_t of delta mass}
    When $\mu_{\ast}$ is the delta mass at a point $y_0 \in \mathbb{R}^d$, then for $t >0$,
    \begin{align*}
        q_t(y) = \frac{1}{{(\sqrt{2\pi} \sigma_t)}^{d}} \cdot e^{-\frac{\|y-y_0\|_2^2}{2{\sigma_t}^2}}, \quad \nabla q_t(y) = -\frac{y-y_0}{\sigma_t^2} q_t(y), \quad \nabla^2 \log q_t(y) = - \frac{1}{\sigma_t^2} \cdot \id_d.
    \end{align*}
    We notice that as $t \to 0^+$, the derivatives of $\log q_t(y)$ blow up.
\end{example}
\begin{example}\label{example: q_t of sphere}
    When $\muast$ is the unit $2$-sphere $\mathbb{S}^2 \subset \mathbb{R}^3 \subset \mathbb{R}^d$, if we write $y=(y',y'') \in \mathbb{R}^3 \times \mathbb{R}^{d-3}$ and $y' = (y_1,y_2,y_3)$, a direct computation shows that
        \begin{align*}
            q_t(y) = \frac{2\pi}{{(\sqrt{2\pi} \sigma_t)}^{d}} \cdot e^{-\frac{\|y\|_2^2 + \lambda_t ^2}{2{\sigma_t}^2}} \cdot \frac{\sigma_t ^2}{\lambda_t\|y'\|_2} \cdot \bigg( e^{\frac{\lambda_t \|y'\|_2 }{\sigma_t ^2}} - e^{-\frac{\lambda_t \|y'\|_2 }{\sigma_t ^2}}\bigg).
        \end{align*}
    The derivatives of $\log q_t(y)$ also blow up as $t \to 0^+$ by a direct computation.
\end{example}

In general, we have the following estimates for space directions derivatives of $q_t(y)$ and $\log q_t(y)$. An interesting fact is that the upper bound in the statement of Lemma~\ref{Lemma: Hessian estimates} is independent of $y \in \mathbb{R}^d$ if the derivatives are of order two or higher.

\begin{lemma}\label{Lemma: Hessian estimates}
   For any $p \geq 3$, any $y\in\mathbb{R}^d$ and any $t >0$, 
        \begin{align}\label{e: all derivatives log qt}
            \|\nabla^p \log q_t (y)\|_{\infty} \leq (4 p!) \frac{ {\lambda_t} ^p D^p}{ {\sigma_t} ^{2p}} .
        \end{align}
    Also, for any $i,j$,
        \begin{align*}
            |\partial_{y_i} \log q_t (y)| \leq \frac{|y_i| +\lambda_t D}{\sigma_t ^2}, \text{ and } |\partial^2 _{y_i y_j} \log q_t (y)| \leq \frac{\delta_{ij}}{\sigma_t ^2} +  2\frac{\lambda_t ^2 D^2}{\sigma_t ^4}.
        \end{align*}
\end{lemma}

\begin{proof}
    For simplicity, for a function $f(x)$ defined on $\mathbb{R}^d$, we denote 
        \begin{align}\label{e: definition of bracket for integral}
            \langle f(x) \rangle \coloneq \int_{\mathbb{R}^d}  f(x) \ \de (\lambda_t)_{\#} \muast(x)
            =\int_{\mathbb{R}^d}  f(x) \muast\bigg(\frac{x}{\lambda_t}\bigg) \frac{1}{{\lambda_t}^d} \ \de x.
        \end{align}
    Hence, $q_t(y) = \langle \frac{1}{{(\sqrt{2\pi} \sigma_t)}^{d}} \cdot e^{-\frac{\|y-x\|_2^2}{2{\sigma_t}^2}} \rangle$.
    For simplicity, we discuss the derivatives of the logarithm of the $y$-function $h(y)\coloneq \langle e^{a{ \| y+x \|_2 ^2} } \rangle$, where $a \in \mathbb{R}$ is a constant. We will choose $a = -\frac{1}{2 \sigma_t ^2}$ finally, and the difference between $\|y-x\|_2^2$ and $\|y+x\|_2^2 = \|-y-x\|_2^2$ doesn't influence the $\|\cdot \|_{\infty}$-norms of derivatives.

    We first compute the first and second derivatives of $\log h (y)$. We notice that 
        \begin{align}\label{eqn:first derivative of h(y)}
            \partial_i h (y) = 2a\langle (y_i+x_i) e^{a{ \| y+x \|_2} ^2} \rangle = 2ay_i h(y) + 2a \langle x_i e^{a{ \| y+x \|_2} ^2} \rangle.
        \end{align}
    So, 
        \begin{align}\label{eqn:second derivative of h(y)}
            \begin{split}
                    \partial^2 _{ij} h (y) &= 2a \delta_{ij} h(y) + 2a y_j(2ay_i h(y) + 2a \langle x_i e^{a{ \| y+x \|_2} ^2} \rangle) + 4a^2 \langle (x_i +y_i)x_j e^{a{ \| y+x \|_2} ^2} \rangle
                \\  &= 2a \delta_{ij} h(y) + 4a^2 y_i y_j h(y) + 4a^2 y_i \langle (x_j e^{a{ \| y+x \|_2} ^2} \rangle + 4a^2 y_j \langle x_i e^{a{ \| y+x \|_2} ^2} \rangle+ 4a^2 \langle x_i x_j e^{a{ \| y+x \|_2} ^2} \rangle.
            \end{split}
        \end{align}
    Hence,
        \begin{align}\label{e: second derivative of log qt}
           \begin{split}
                \partial^2 _{ij} \log h(y) &= \frac{\partial^2 _{ij} h (y) h (y) - \partial_i h (y) \partial_j h (y)}{h(y) ^2} \\ &= 2a \delta_{ij} + 4a ^2 \frac{\langle x_i x_j e^{a{ \| y+x \|_2} ^2} \rangle h(y) - \langle x_i e^{a{ \| y+x \|_2} ^2} \rangle \langle x_j e^{a{ \| y+x \|_2} ^2} \rangle}{h(y) ^2}.
           \end{split}
        \end{align}
    Before we compute higher order derivatives of $\log h(y)$, we first illustrate how we obtain an upper bound for $\partial^2 _{ij} \log h(y)$. Notice that in the definition of $\langle f(x) \rangle$ in \eqref{e: definition of bracket for integral}, because $\muast$ has a compact support $K_{\ast}$, the $x$ in the integrand satisfies that $x \in \lambda_t K_{\ast}$. Hence, in the integrand of \eqref{e: definition of bracket for integral}, $\|x\|_{\infty} \leq \lambda_t D$ by Assumption~\ref{assumption:secon-moment}. So, $|\langle x_i e^{a{ \| y+x \|_2} ^2} \rangle| \leq \lambda_t D h(y)$, $|\langle x_i x_j e^{a{ \| y+x \|_2} ^2} \rangle| \leq \lambda_t ^2 D^2 h(y)$. So, $|\partial^2 _{ij} \log h(y)| \leq 2a \delta_{ij} + 8a^2 \lambda_t ^2 D^2 $.

    Then, we compute $\nabla^p \log q_t (y)$ for $p \geq 3$. Before we start, we remark that an important observation is that, in \eqref{e: second derivative of log qt}, there is no polynomial term of $y_i,y_j$ like we have seen in \eqref{eqn:first derivative of h(y)} and \eqref{eqn:second derivative of h(y)}. We will then use induction to get the formula of $\nabla ^p \log h(y)$ with $p \geq 2$. Assume that for an $m \in \mathbb{Z}_+$, and for any multi-index $\alpha$ with $|\alpha| = m$, the derivative $\partial_y ^{\alpha} \log h(y)$ has a form 
        \begin{align*}
            \partial_y ^{\alpha} \log h(y) = {(2a)}^{m} \frac{P_{\alpha} (y)}{ h(y) ^m},
        \end{align*}
    where $P_{\alpha} (y)$ is the $\mathbb{R}$-linear combination of at most $4 (m-1) ! - 2$ terms, where each term has a form
        \begin{align*}
            \pm \langle x^{\beta_1} e^{a{ \| y+x \|_2} ^2} \rangle \cdot \langle x^{\beta_2} e^{a{ \| y+x \|_2} ^2} \rangle  \cdots \langle x^{\beta_m} e^{a{ \| y+x \|_2} ^2} \rangle,
        \end{align*}
    and each $\beta_i$ is a multi-index and $|\beta_1| + |\beta_2| + \ldots + |\beta_m| = m$. This assumption is satisfied when $m=2$ according to the explicit form of $\partial^2 _{ij} \log h(y)$ shown in \eqref{e: second derivative of log qt}, excluding the term $2a \delta_{ij}$. This $2a \delta_{ij}$ term in \eqref{e: second derivative of log qt} can be safely excluded for the induction because its derivative is $0$. Then,
        \begin{align}\label{e: one more step derivative qt}
            \partial_{y_1}\partial_y ^{\alpha} \log h(y) = {(2a)}^{m} \partial_{y_1} \frac{P_\alpha (y)}{ h(y) ^m} = {(2a)}^{m+1} \cdot \frac{ \frac{1}{2a}\partial_{y_1} P_\alpha (y) h(y) - m P_\alpha (y) \langle (y_1+x_1) e^{a{ \| y+x \|_2} ^2 } \rangle }{h(y) ^{m+1}}.
        \end{align}
    We notice that
        \begin{align*}
            \begin{split}
                &\frac{1}{2a}\partial_{y_1}  \big( \langle x^{\beta_1} e^{a{ \| y+x \|_2} ^2} \rangle \cdot \langle x^{\beta_2} e^{a{ \| y+x \|_2} ^2} \rangle  \cdots \langle x^{\beta_m} e^{a{ \| y+x \|_2} ^2} \rangle \big) 
                \\ &= m y_1\langle x^{\beta_1} e^{a{ \| y+x \|_2} ^2} \rangle \cdot \langle x^{\beta_2} e^{a{ \| y+x \|_2} ^2} \rangle  \cdots \langle x^{\beta_m} e^{a{ \| y+x \|_2} ^2} \rangle 
                \\ & \quad + \langle x_1 x^{\beta_1} e^{a{ \| y+x \|_2} ^2} \rangle \cdot \langle x^{\beta_2} e^{a{ \| y+x \|_2} ^2} \rangle  \cdots \langle x^{\beta_m} e^{a{ \| y+x \|_2} ^2} \rangle
            \\ &\quad + \langle x^{\beta_1} e^{a{ \| y+x \|_2} ^2} \rangle \cdot \langle x_1 x^{\beta_2} e^{a{ \| y+x \|_2} ^2} \rangle  \cdots \langle x^{\beta_m} e^{a{ \| y+x \|_2} ^2} \rangle 
            \\ & \quad + \ldots + \langle x^{\beta_1} e^{a{ \| y+x \|_2} ^2} \rangle \cdot \langle x^{\beta_2} e^{a{ \| y+x \|_2} ^2} \rangle  \cdots \langle x_1 x^{\beta_m} e^{a{ \| y+x \|_2} ^2} \rangle.
            \end{split}
        \end{align*}
    Hence, $\frac{1}{2a}\partial_{y_1} P_\alpha (y)$ is of the form $ m y_1 P_\alpha (y) + Q_\alpha(y)$, where $Q_\alpha(y)$ is the summation of at most $4 (m-1) ! - 2$ terms without $y_1$ showing in the polynomial terms.
    We notice that those $y_1$ terms in the numerator of \eqref{e: one more step derivative qt} will then cancel, and the remaining terms all look like
        \begin{align*}
            \pm \langle x^{\beta_1 ' } e^{a{ \| y+x \|_2} ^2} \rangle \cdot \langle x^{\beta_2 ' } e^{a{ \| y+x \|_2} ^2} \rangle  \cdots \langle x^{\beta_{m+1}'} e^{a{ \| y+x \|_2} ^2} \rangle,
        \end{align*}
    where each $\beta_i '$ is a multi-index and $|\beta_1 ' | + |\beta_2 '| + \ldots + |\beta_{m+1} '| = m+1$. The number of them is at most $m (4 (m-1) ! - 2) +m = 4 m! - m \leq  4 m! - 2$ when $m \geq 2$. We then conclude the proof of Lemma~\ref{Lemma: Hessian estimates}.
\end{proof}

The following lemma describes the situation when $t >0$ is very large. We remark that one can also use Lemma 9 in \cite{chen2023improved} to obtain a similar result as \Cref{lemma: qt close to Gaussian}.
\begin{lemma}\label{lemma: qt close to Gaussian}
    Let $q_t(y)$ be defined in \eqref{e: definition of q_t}, and let
        \begin{align*}
            g(y) \coloneq \frac{1}{{(\sqrt{2\pi} )}^{d}} \cdot e^{-\frac{\|y\|_2^2}{2}}.
        \end{align*}
    Then, we have that for $t >0$,
        \begin{align}\label{e:qt-q_bound}
            \int_{\mathbb{R}^d} | q_t(y)  - g(y)| \de y \leq \frac{2d(1-\sigma_t)}{{\sigma_t}^{d+3}} + \lambda_t \sqrt{d}D \frac{4+ \lambda_t \sqrt{d}D}{\sigma_t ^2} \cdot e^{\frac{ \lambda_t ^2 d D^2}{2{\sigma_t}^2}}.
        \end{align}
    %In particular, this difference goes to $0$ exponentially as $t \to +\infty$.
 In particular, there exists a universal constant $C_u>0$, such that
\begin{align}\label{e:simplified_bound}
         \int_{\mathbb{R}^d} | q_t(y)  - g(y)| \de y \leq C_u e^{-t}\sqrt d D,
    \end{align}
    which goes to $0$ exponentially as $t\rightarrow +\infty$.
    
\end{lemma}
\begin{proof}
    Because one can write 
        \begin{align*}
            q_t (y) =\int_{\mathbb{R}^d} \frac{1}{{(\sqrt{2\pi} \sigma_t)}^{d}} \cdot e^{-\frac{\|y-\lambda_t x\|_2^2}{2{\sigma_t}^2}} \ d\muast(x),
        \end{align*}
    and $\int_{\mathbb{R}^d} d\muast(x) = 1$, we see that
        \begin{align}\label{e: difference of qt and Gaussian}
            \begin{split}
                \int_{\mathbb{R}^d}| q_t(y)  - g(y)| \de y &= \int_{\mathbb{R}^d} \bigg| \int_{\mathbb{R}^d} \frac{1}{{(\sqrt{2\pi} \sigma_t)}^{d}} \cdot e^{-\frac{\|y-\lambda_t x\|_2^2}{2{\sigma_t}^2}} - \frac{1}{{(\sqrt{2\pi} )}^{d}} \cdot e^{-\frac{\|y\|_2^2}{2}} \ d\muast(x) \bigg| \de y
                \\  &\leq \int_{\mathbb{R}^d} \int_{\mathbb{R}^d} \bigg| \frac{1}{{(\sqrt{2\pi} \sigma_t)}^{d}} \cdot e^{-\frac{\|y-\lambda_t x\|_2^2}{2{\sigma_t}^2}} - \frac{1}{{(\sqrt{2\pi} )}^{d}} \cdot e^{-\frac{\|y\|_2^2}{2}} \bigg| \de y \  d\muast(x).
            \end{split}
        \end{align}
    We denote $y = (y_1,y') \in \mathbb{R}^d$, and we define the function 
        \begin{align*}
            f(y,r)\coloneq \bigg| \frac{1}{{(\sqrt{2\pi} \sigma_t)}^{d}} \cdot e^{-\frac{{|y_1-\lambda_t r|}^2 + {\|y'\|_2^2}}{2{\sigma_t}^2}} - \frac{1}{{(\sqrt{2\pi} )}^{d}} \cdot e^{-\frac{\|y\|_2^2}{2}} \bigg|.
        \end{align*}
    By the rotation symmetry, we see that the right hand side of \eqref{e: difference of qt and Gaussian} equals to
        \begin{align}\label{e: difference of qt and Gaussian simplified 1}
            \int_{\mathbb{R}^d} \int_{\mathbb{R}^d} f(y, \|x\|_2) \de y \  d\muast(x).
        \end{align}
    Let's first see what is $|\frac{\de}{\de r} f(y,r)|$ for $r \leq \sqrt{d}D$. A direct computation shows that 
        \begin{align*}
            \begin{split}
                \bigg|\frac{\de}{\de r} f(y,r)\bigg| &= \frac{\lambda_t}{{(\sqrt{2\pi} \sigma_t)}^{d}} \cdot \frac{|y_1 - \lambda_t r|}{\sigma_t ^2} \cdot e^{-\frac{{|y_1-\lambda_t r|}^2 + {\|y'\|_2^2}}{2{\sigma_t}^2}}
                \\  &\leq \frac{\lambda_t}{{(\sqrt{2\pi} \sigma_t)}^{d}} \cdot \frac{|y_1| + \lambda_t \sqrt{d}D}{\sigma_t ^2} \cdot e^{-\frac{{\frac{1}{2}|y_1|}^2 - \lambda_t ^2 d D^2 + {\|y'\|_2^2}}{2{\sigma_t}^2}}.
            \end{split}
        \end{align*}
    So, \eqref{e: difference of qt and Gaussian simplified 1} is bounded by
        \begin{align}\label{e: difference of qt and Gaussian simplified 2}
            \int_{\mathbb{R}^d} \int_{\mathbb{R}^d} f(y, 0) \de y \  d\muast(x) + \int_{\mathbb{R}^d} \int_{\mathbb{R}^d}  \frac{\lambda_t \|x\|_2}{{(\sqrt{2\pi} \sigma_t)}^{d}} \cdot \frac{|y_1| + \lambda_t \sqrt{d}D}{\sigma_t ^2} \cdot e^{-\frac{{\frac{1}{2}|y_1|}^2 - \lambda_t ^2 d D^2 + {\|y'\|_2^2}}{2{\sigma_t}^2}} \de y \  d\muast(x).
        \end{align}
        
    The second term in \eqref{e: difference of qt and Gaussian simplified 2} is bounded by  
        \begin{align}\label{e:bound1}
             \int_{\mathbb{R}^d}  \frac{\lambda_t \sqrt{d}D}{{(\sqrt{2\pi} \sigma_t)}^{d}} \cdot \frac{|y_1| + \lambda_t \sqrt{d}D}{\sigma_t ^2} \cdot e^{-\frac{{\frac{1}{2}|y_1|}^2 - \lambda_t ^2 d D^2 + {\|y'\|_2^2}}{2{\sigma_t}^2}} \de y \leq \lambda_t \sqrt{d}D \frac{4+ \lambda_t \sqrt{d}D}{\sigma_t ^2} \cdot e^{\frac{ \lambda_t ^2 d D^2}{2{\sigma_t}^2}},
        \end{align}
    which goes to $0$ exponentially as $t \to +\infty$ because $\lambda_t = e^{-t}$ and $\sigma_t ^2 = 1 - \lambda_t ^2$.

    The first term in \eqref{e: difference of qt and Gaussian simplified 2} is 
        \begin{align}\label{e: difference of two Gaussian}
             \int_{\mathbb{R}^d} \bigg| \frac{1}{{(\sqrt{2\pi} \sigma_t)}^{d}} \cdot e^{-\frac{ {\|y\|_2^2}}{2{\sigma_t}^2}} - \frac{1}{{(\sqrt{2\pi} )}^{d}} \cdot e^{-\frac{\|y\|_2^2}{2}} \bigg| \de y,
        \end{align}
    because $\int_{\mathbb{R}^d} d\muast(x) = 1$. We define the function
        \begin{align*}
            h(y,\alpha) = \bigg| \frac{1}{{(\sqrt{2\pi} \alpha)}^{d}} \cdot e^{-\frac{ {\|y\|_2^2}}{2{\alpha}^2}} - \frac{1}{{(\sqrt{2\pi} )}^{d}} \cdot e^{-\frac{\|y\|_2^2}{2}} \bigg|,
        \end{align*}
    and then for $\alpha \in [0,1]$,
        \begin{align*}
            \bigg|\frac{\de}{\de\alpha} h (y,\alpha)\bigg| = \frac{1}{{(\sqrt{2\pi} )}^{d}}\bigg| -d \frac{1}{\alpha^{d+1}} e^{-\frac{ {\|y\|_2^2}}{2{\alpha}^2}} + \frac{1}{\alpha ^d} e^{-\frac{ {\|y\|_2^2}}{2{\alpha}^2}} \cdot \frac{\|y\|_2 ^2}{\alpha^3} \bigg| \leq \frac{e^{-\frac{ {\|y\|_2^2}}{2{\alpha}^2}}}{{(\sqrt{2\pi} )}^{d} \alpha^{d+3}}  \cdot |d+ {\| y \|_2}^2 |.
        \end{align*}
    Because $h(y,1) = 0$, we see that \eqref{e: difference of two Gaussian} is bounded by
        \begin{align}\label{e:bound2}
            \frac{(1-\sigma_t)}{{\sigma_t}^{d+3}}\int_{\mathbb{R}^d} \frac{e^{-\frac{ {\|y\|_2^2}}{2}}}{{(\sqrt{2\pi} )}^{d} }  \cdot |d+ {\| y \|_2}^2 | \de y = \frac{2d(1-\sigma_t)}{{\sigma_t}^{d+3}},
        \end{align}
    which goes to $0$ exponentially as $t \to +\infty$ because $\sigma_t = \sqrt{1-e^{-2t}}$.    The estimates \eqref{e:bound1} and \eqref{e:bound2} together give
    \eqref{e:qt-q_bound}. Next we show \eqref{e:simplified_bound}. There are two cases: if $e^{-t}\sqrt d D\geq 1$, then \eqref{e:simplified_bound} holds trivially, since the lefthand side is at most $2$. Otherwise $e^{-t}\sqrt d D\leq 1$, we have
    \begin{align*}\begin{split}
        &\phantom{{}={}}\frac{2d(1-\sigma_t)}{{\sigma_t}^{d+3}} + \lambda_t \sqrt{d}D \frac{4+ \lambda_t \sqrt{d}D}{\sigma_t ^2} \cdot e^{\frac{ \lambda_t ^2 d D^2}{2{\sigma_t}^2}}
        \leq \frac{2d e^{-2t}}{\sigma_t^{2d}}
        +\frac{4 e^{-t}\sqrt d D+(e^{-t}\sqrt d D)^2}{\sigma_t^2} \cdot e^{\frac{1}{2{\sigma_t}^2}} \\
        &\leq\frac{2 de^{-2t}}{\sigma_t^{2d}}+ \frac{5 e^{-t}\sqrt d D}{\sigma_t^2}  e^{\frac{2}{3}}
        =\frac{2 d e^{-2t}}{(1-e^{-2t})^{d}}+ \frac{5 e^{-t}\sqrt d D}{1-e^{-2t}} e^{\frac{2}{3}}
        \\ &
        \leq \frac{2 \sqrt{d} e^{-t}}{1-d e^{-2t}} \cdot \sqrt{d} e^{-t}+ \frac{5 e^{-t}\sqrt d D}{1-e^{-2t}}e^{\frac{2}{3}}\leq C_u e^{-t} \sqrt{d}D,
    \end{split}\end{align*}
    for a universal constant $C_u >0$.
    In the last inequality, we assume that $e^{-t} \leq \frac{1}{4\sqrt{d}} \leq \frac{1}{4}$. Otherwise, if $e^{-t} \geq \frac{1}{4\sqrt{d}}$, \eqref{e:simplified_bound} also holds trivially.
    This finishes the proof of \eqref{e:simplified_bound}.
\end{proof}

When we prove our main theorem, Theorem~\ref{theorem: main L^1 theorem}, in the following Section~\ref{section:score estimation error}, we will assume that $\muast$ has a compact support $K_{\ast}$ and use Lemma~\ref{Lemma: Hessian estimates} to complete the proof. On the other hand, as we have mentioned in Remark~\ref{remark: modify error by stronger assumptions}, our methods work under other assumptions on the initial data $\mu_\ast$, as long as under those assumptions, we can reasonably obtain the properties in Remark~\ref{remark: other assumptions on initial data}. We next assume that $\muast$ is a Gaussian mixture and obtain an estimate similar to Lemma~\ref{Lemma: Hessian estimates}.
\begin{lemma}\label{lemma: Hessian estimates of Gaussian mixture initial data}
    Assume that $\muast$ is a Gaussian mixture, i.e., we assume that
        \begin{align*}
            \muast(x) \coloneq \sum_{k=1} ^M c_k \frac{1}{{(\sqrt{2\pi})}^d  a_k} \cdot \exp{\bigg(-\frac{1}{2} (x-b_k) \cdot  A_k ^{-1} (x-b_k) \bigg)},
        \end{align*}
    where $A_k$'s are positive definite matrices, $a_k \coloneq \sqrt{\det A_k}$, $b_k$'s are constant vectors in $\mathbb{R}^d$, $c_k >0$ and $\sum_{k=1} ^M c_k = 1$. Then, for any $\ell \in \mathbb{Z}_+$ and any multi-index $\alpha$ with $|\alpha| \leq \ell$, there is a constant $C(\ell, \muast)$, depending on $\ell$ and $A_k,b_k$ in the formula of $\muast$, such that for any $x \in \mathbb{R}^d$ and any $t \geq 0$,
        \begin{align}\label{e: gradient estimate Gaussian mixture}
            \frac{|\partial_x ^{\alpha} q_t(x)|}{q_t(x)} \leq C(\ell, \muast) \cdot \|x\|_2^{\ell} + C(\ell, \muast).
        \end{align}

\end{lemma}
\begin{proof}
    A standard computation, by \eqref{e: definition of q_t}, shows that the density function of $q_t(x)$ is 
    \begin{align*}
        q_t(x) \coloneq \sum_{k=1} ^M c_k \frac{1}{{(\sqrt{2\pi})}^d \cdot a_k(t)} \cdot \exp{\bigg(-\frac{1}{2} (x-b_k(t)) \cdot  A_k(t) ^{-1} (x-b_k(t)) \bigg)},
    \end{align*}
where $A_k(t) = \lambda_t ^2 A_k + \sigma_t ^2 \id_d$, $b_k(t) = \lambda_t b_k$. Hence, 
    \begin{align*}
        \nabla q_t(x) =\sum_{k=1} ^M c_k \frac{-A_k(t) ^{-1} (x-b_k(t))}{{(\sqrt{2\pi})}^d \cdot a_k(t)} \cdot \exp{\bigg(-\frac{1}{2} (x-b_k(t)) \cdot  A_k(t) ^{-1} (x-b_k(t)) \bigg)} .
    \end{align*}
So,
    \begin{align*}
        |\partial_1 q_t(x)| \leq \max_{1 \leq k \leq M} \|-A_k(t) ^{-1} (x-b_k(t))\| \cdot q_t(x).
    \end{align*}
Notice that because when $t=0$, $A_k$'s are positive definite, and when $t \to +\infty$, $A_k(t) \to \id_d$, there is an upper bound $C(\ell,\muast)$, such that $\sup_{t \in [0,+\infty)} \|A_k(t) ^{-1}\| \leq C(\ell,\muast)$. Hence, we can obtain \eqref{e: gradient estimate Gaussian mixture} for $\ell =1$. For general $\ell \geq 1$, one can either use induction or compute them directly, exactly using the same way as the case $\ell=1$ we have shown. For the purpose of proving Theorem~\ref{theorem: main L^1 theorem}, the estimates for $\ell \leq 3$ will be enough.
    
\end{proof}

Next, we also point out that those $\widehat q_{t}$'s also have exponential tails, as long as one of them has an exponential tail at a given time. For example, if $\widehat q_T = q_T$.
\begin{lemma}\label{lemma: hat qt exponential tail}
    If for a $t'>0$, $\widehat q_{t'} (x)$ has an exponential tail as $\| x \|_2 \to + \infty$, then for any $t >0$, $\widehat q_t(x)$ also has an exponential tail as $\| x \|_2 \to + \infty$.
\end{lemma}
\begin{proof}
    Notice that $\del_t \widehat q_t =\nabla \cdot (\sU_t
    \widehat q_t)$. Hence, along the characteristic line $\frac{\de}{\de t} T_t(x) = - \sU_{t'+t}(T_t(x))$, one can solve that 
        \begin{align*}
            \widehat q_{t' + t}(T_t(x)) = \widehat q_{t'}(x) \cdot \exp{\bigg( \int_{0} ^t (\nabla \cdot \sU_{t' + s})(T_s(x)) \ \de s \bigg)}.
        \end{align*}
    Because $\sU_t(x) = x + s_{T-t}(x)$, by Assumption~\ref{a:score-derivative}, we have that $|(\nabla \cdot \sU_{t})(x) | \leq d(1+ L_{T-t})$ for any $t >0$. The remaining estimate is on the norm of $T_t(x)$. We notice that, by Assumption~\ref{a:score-derivative} again,
    \begin{align*}
        \|s_{T-t}(x)\|_2 \leq \|s_{T-t}(0)\|_2 + (\|x\|_2 \cdot d \cdot L_{T-t}) \leq d\cdot  L_{T-t} (1 + \|x\|_2).
    \end{align*}
    Hence,
        \begin{align*}
           \begin{split}
                |\del_t \| T_t(x) \|_2 ^2 |= |\sU_{t'+t}(T_t(x)) \cdot T_t(x)  | &\leq (1+ dL_{T-t'-t})\| T_t(x) \|_2 ^2 + d L_{T-t'-t} \| T_t(x) \|_2 \\ &\leq (2+ dL_{T-t'-t}) \| T_t(x) \|_2 ^2 + d^2L_{T-t'-t} ^2.
           \end{split}
        \end{align*}
    Denote $\mathcal{L}_t = \int_0 ^{t} (2+ d\cdot L_{T-t'-s}) \ \de s$.
    Hence, for $t >0$, 
        \begin{align*}
            e^{-\mathcal{L}_t} \cdot \bigg(\|x\|_2^2 - \int_0 ^t e^{\mathcal{L}_s}\cdot d^2 \cdot L_{T-t'-s} ^2 \ \de s \bigg) \leq \| T_t(x) \|_2 ^2 \leq e^{\mathcal{L}_t} \cdot \bigg(\|x\|_2^2 + \int_0 ^t e^{-\mathcal{L}_s} \cdot d^2 \cdot  L_{T-t'-s} ^2 \ \de s\bigg).
        \end{align*}
    Because we can compare the norms of $T_t(x)$ and $x$ by a factor only depending on $t$, and $T_t$ is also a diffeomorphism since $\sU_t(x)$ is locally Lipschitz, then we know that $\widehat q_t(x)$ also has an exponential tail as $\|x\|_2 \to + \infty$.
    One can obtain a similar result for $t <0$. 
\end{proof}

\section{Score Estimation Error}\label{section:score estimation error}

In this section, we assume that $\muast$ has a compact support $K_{\ast}$ as in Assumption~\ref{assumption:secon-moment} to proceed the proof first. At the end of this section, we will point out some possible ways to use other assumptions on this initial data $\muast$ in Remark~\ref{remark: other assumptions on initial data}.
We first need the following Gagliardo-Nirenberg interpolation inequality to estimate $E(t)$ defined in Theorem~\ref{theorem: L^1 error}, and finally use $\int_{\mathbb{R}^d} q_{t}(x) \delta_{t} ^2 (x) \ \de x$ to control $E(t)$ when $t \in [\mft,T]$. For the convenience of readers, let us also sketch the proof of this inequality here.
\begin{lemma}[Gagliardo-Nirenberg]\label{lemma: Gagliardo-Nirenberg}
    There is a positive universal constant $C_u$, such that for any $d \in \mathbb{Z}_{+}$, any $w \in L^1$, any $i \in \llbracket 1,d\rrbracket$, if $\partial^2 _{ii} w \in L^1$, then
        \begin{align}\label{e: Gagliardo k=2}
        {\bigg(\int_{\mathbb{R}^d}|\partial_i w(x)| \ \de x\bigg)}^{2} \leq C_u \bigg(\int_{\mathbb{R}^d}|\partial^2 _{ii} w(x)| \ \de x\bigg)\bigg(\int_{\mathbb{R}^d}|w(x)| \ \de x\bigg).
        \end{align}
In general, if $\partial^k _{i} w \in L^1$ with $k \geq 2$, then
        \begin{align}\label{e: Gagliardo general k}
        \int_{\mathbb{R}^d}|\partial_i w(x)| \ \de x \leq C_u ^{\frac{k-1}{2}}{\bigg(\int_{\mathbb{R}^d}|\partial^k _{i} w(x)| \ \de x\bigg)}^{\frac{1}{k}} {\bigg(\int_{\mathbb{R}^d}|w(x)| \ \de x\bigg)}^{\frac{k-1}{k}}.
        \end{align}
\end{lemma}
\begin{proof}
    Without loss of generality, we assume that $i=1$. For any $u \in C^2 _c(\mathbb{R}^d)$, we fix its remaining coordinates $x' = (x_2, \ldots, x_d)$, then according to Lemma~3.4 of~\cite{Fiorenza2021Gagliardo}, there is a universal constant $C_u >0$, such that
        \begin{align*}
        {\bigg(\int_{\mathbb{R}}|\partial_1 w(x_1,x')| \ \de x_1 \bigg)}^{2} \leq C_u \bigg(\int_{\mathbb{R}}|\partial^2 _{11} w(x_1,x')| \ \de x_1\bigg)\bigg(\int_{\mathbb{R}}|w(x_1, x')| \ \de x_1\bigg).
        \end{align*}
    Hence,
        \begin{align*}
            \begin{split}
                \bigg(\int_{\mathbb{R}^d}|\partial_i w(x)| \ \de x\bigg) &= \bigg(\int_{\mathbb{R}^{d-1}} \int_{\mathbb{R}}|\partial_1 w(x_1,x')| \ \de x_1 \de x'\bigg) 
                \\ &\leq C_u ^{\frac{1}{2}} \int_{\mathbb{R}^{d-1}}{\bigg(\int_{\mathbb{R}}|\partial^2 _{11} w(x_1,x')| \ \de x_1\bigg) }^{\frac{1}{2}}{\bigg(\int_{\mathbb{R}}|w(x_1, x')| \ \de x_1\bigg)}^{\frac{1}{2}} \ \de x'
                \\ &\leq C_u ^{\frac{1}{2}} {\bigg(\int_{\mathbb{R}^{d-1}}\int_{\mathbb{R}}|\partial^2 _{11} w(x_1,x')| \ \de x_1 \ \de x'\bigg) }^{\frac{1}{2}}
                {\bigg(\int_{\mathbb{R}^{d-1}}\int_{\mathbb{R}}| w(x_1,x')| \ \de x_1 \ \de x'\bigg) }^{\frac{1}{2}} .
            \end{split}
        \end{align*}
In general, assume that we already know that for some $k \geq 2$, 
    \begin{align*}
        \int_{\mathbb{R}^d}|\partial_i w(x)| \ \de x \leq C_u ^{\frac{k-1}{2}}{\bigg(\int_{\mathbb{R}^d}|\partial^k _{i} w(x)| \ \de x\bigg)}^{\frac{1}{k}} {\bigg(\int_{\mathbb{R}^d}|w(x)| \ \de x\bigg)}^{\frac{k-1}{k}},
    \end{align*}
then we replace $w(x)$ with  $\partial_i w(x)$, and obtain that 
    \begin{align*}
        \int_{\mathbb{R}^d}|\partial^2 _{ii} w(x)| \ \de x \leq C_u ^{\frac{k-1}{2}}{\bigg(\int_{\mathbb{R}^d}|\partial^{k+1} _{i} w(x)| \ \de x\bigg)}^{\frac{1}{k}} {\bigg(\int_{\mathbb{R}^d}|\partial_{i} w(x)| \ \de x\bigg)}^{\frac{k-1}{k}}.
    \end{align*}
Combining this inequality and the inequality \eqref{e: Gagliardo k=2}, we can obtain \eqref{e: Gagliardo general k} for $k+1$.
\end{proof}

Now, let's estimate the integral of $E(t)$ in $t$. By \eqref{e: Gagliardo k=2} in Lemma~\ref{lemma: Gagliardo-Nirenberg} and H\"{o}lder inequality,
    \begin{align}\label{e: error after gagliardo}
        \begin{split}
            \int_{\mft} ^ T{\bigg(\int_{\mathbb{R}^d}|\partial_1 (q_t \delta_t ^1) (x)| \ \de x\bigg)} \ \de t &\leq C_u ^{\frac{1}{2}} \int_{\mft} ^ T{\bigg(\int_{\mathbb{R}^d}|\partial^2 _{11} (q_t \delta_t ^1) (x)| \ \de x\bigg)}^{\frac{1}{2}}{\bigg(\int_{\mathbb{R}^d}|(q_t \delta_t ^1) (x)| \ \de x\bigg)}^{\frac{1}{2}} \ \de t
            \\ &\leq C_u ^{\frac{1}{2}} {\bigg(\int_{\mft} ^ T \int_{\mathbb{R}^d}|\partial^2 _{11} (q_t \delta_t ^1) (x)| \ \de x \ \de t \bigg)}^{\frac{1}{2}}{\bigg(\int_{\mft} ^T \int_{\mathbb{R}^d}|(q_t \delta_t ^1) (x)| \ \de x \ \de t\bigg)}^{\frac{1}{2}} ,
        \end{split}
    \end{align}
where we use the notation $\delta_t = (\delta_t ^1 , \ldots, \delta_t ^d)$.
For the term
    \begin{align*}
        {\int_{\mft} ^T \int_{\mathbb{R}^d}|(q_t \delta_t ^1) (x)| \ \de x \ \de t},
    \end{align*}
we use H\"{o}lder inequality twice and the fact that $\int_{\mathbb{R}^d}q_t (x) \ \de x =1$, and see that
    \begin{align*}
        \begin{split}
        {\int_{\mft} ^T \int_{\mathbb{R}^d}|(q_t \delta_t ^1) (x)| \ \de x \ \de t} &\leq {\int_{\mft} ^T {\bigg(\int_{\mathbb{R}^d}q_t(x) {(\delta_t ^1(x))}^2 \ \de x\bigg)} ^{\frac{1}{2}} {\bigg(\int_{\mathbb{R}^d}q_t (x) \ \de x \bigg)}^\frac{1}{2} \ \de t}
        \\  &\leq
        {\bigg(\int_{\mft} ^T \int_{\mathbb{R}^d}q_t(x) {(\delta_t ^1(x))}^2 \ \de x \ \de t \bigg) ^{\frac{1}{2}}  } \cdot {(T-\mft)}^{\frac{1}{2}}.
        \end{split}
    \end{align*}
Notice that our assumption is that $\int_{\mft} ^T\int_{\mathbb{R}^d}q_t(x) {(\delta_t ^1(x))}^2 \ \de x\de t$ can be made very small.
Next, we are going to show that the term 
    \begin{align}\label{e: hessian after gagliardo}
    \int_{\mft} ^ T \int_{\mathbb{R}^d}|\partial^2 _{11} (q_t \delta_t ^1) (x)| \ \de x \ \de t 
    \end{align}
is bounded by a positive constant depending on $\mft, T , \mathcal{L}, D$.
For the terms in $\partial^2 _{11} (q_t \delta_t ^1) (x) = (\partial^2 _{11} q_t) \delta_t ^1 + 2 \partial _{1} q_t \partial _{1} \delta_t ^1 + q_t (\partial^2 _{11} \delta_t ^1) $, we notice that, because $\delta_t(x) =  s_{T-t}(x) - \nabla\log q_t (x)$, by the proof of Lemma~\ref{Lemma: Hessian estimates}
\begin{align*}
    |\partial _{1} \delta_t ^1 (x)| \leq L_{T-t} + \frac{1}{ {\sigma_t}^4} (\sigma_t ^2  + 2 \lambda_t ^2 D^2 ),
\end{align*}

\begin{align*}
    |\partial^2 _{11} \delta_t ^1 (x)| \leq L_{T-t} + \frac{24{\lambda_t} ^3}{ {\sigma_t} ^6} D^3,
\end{align*}
and we see that for any $x = (x_1,x_2, \ldots, x_d) \in \mathbb{R}^d$, by the proof of Lemma~\ref{Lemma: Hessian estimates} again,
    \begin{align}\label{e: gradient log q_t}
        \bigg| \frac{\partial _{1} q_t(x)}{q_t(x)} \bigg| \leq \frac{(|x_1| + \lambda_t D) }{\sigma_t ^2},
    \end{align}
    
    \begin{align*}
        \bigg| \frac{\partial^2 _{11} q_t(x)}{q_t(x)} \bigg| \leq \frac{2({|x_1|}^2 + \lambda_t ^2 D^2) + \sigma_t^2}{\sigma_t ^4}.
    \end{align*}
Hence,
    \begin{align*}
       \begin{split}
            {\bigg(\int_{\mathbb{R}^d}|(\partial^2 _{11} q_t)(x) \delta_t ^1 (x)| \ \de x \bigg)}^2 
            &= {\bigg(\int_{\mathbb{R}^d} \bigg|\frac{(\partial^2 _{11} q_t)(x) }{q_t(x)}  \delta_t ^1 (x) q_t(x) \bigg| \ \de x \bigg)}^2 
            \\ &\leq \bigg(\int_{\mathbb{R}^d} {\bigg( \frac{2({|x_1|}^2+ \lambda_t ^2 D^2) + \sigma_t^2}{\sigma_t ^4} \bigg)}^2 q_t(x) \ \de x \bigg) \cdot \bigg(\int_{\mathbb{R}^d} {(\delta_t ^1 (x))}^2 q_t(x) \ \de x \bigg),
       \end{split}
    \end{align*}
where the first term is a bounded term by similarly analyzing $q_t(x)$. For example, let us show that the $x_1$-fourth moment of $q_t$, i.e., $\int_{\mathbb{R}^d} {|x_1|}^4 q_t(x) \ \de x$, is bounded. By \eqref{e: definition of q_t}, we know that
    \begin{align}\label{e: fourth moment of q_t}
        \begin{split}
            \int_{\mathbb{R}^d} {|y_1|}^4 q_t(y) \ \de y &= \int_{\mathbb{R}^d} \int_{\mathbb{R}^d} \frac{{|y_1|}^4}{{(\sqrt{2\pi} \sigma_t)}^{d}} \cdot e^{-\frac{\|y-x\|_2^2}{2{\sigma_t}^2}} \cdot \muast\bigg(\frac{x}{\lambda_t}\bigg) \frac{1}{{\lambda_t}^d} \ \de x \de y
            \\  &\leq 8\int_{\mathbb{R}^d} \int_{\mathbb{R}^d} \frac{{|y_1-x_1|}^4 + {|x_1|}^4}{{(\sqrt{2\pi} \sigma_t)}^{d}} \cdot e^{-\frac{\|y-x\|_2^2}{2{\sigma_t}^2}} \cdot \muast\bigg(\frac{x}{\lambda_t}\bigg) \frac{1}{{\lambda_t}^d} \ \de x \de y
            \\  &\leq 8\int_{\mathbb{R}^d} \int_{\mathbb{R}^d} \frac{{|y_1-x_1|}^4 + \lambda_t ^4 D^4}{{(\sqrt{2\pi} \sigma_t)}^{d}} \cdot e^{-\frac{\|y-x\|_2^2}{2{\sigma_t}^2}} \cdot \muast\bigg(\frac{x}{\lambda_t}\bigg) \frac{1}{{\lambda_t}^d} \ \de x \de y
            \\  &=C_u (\sigma_t ^4 + \lambda_t ^4 D^4) \leq C_u(1+ D^4),
        \end{split}
    \end{align}
for a universal constant $C_u >0$ coming from the fourth moment of the standard Gaussian. We can similarly estimate the remaining two terms in the expansion of $\partial^2 _{11} (q_t \delta_t ^1) (x)$ and get an upper bound for \eqref{e: hessian after gagliardo}.
Also, we remark that after taking the time integral from $\mft$ to $T$, the main order of $\sigma_t$ involved in \eqref{e: hessian after gagliardo} is at most
    \begin{align*}
        \int_{\mft} ^T \frac{1}{\sigma_t ^6} \ \de t = \int_{\mft} ^T \frac{e^{6t}}{{(e^{2t}-1)}^3} \ \de t,
    \end{align*}
which blows up of order $T$ as $T \to +\infty$ and blows up of order $\mft^{-2}$ as $\mft \to 0^+$. Hence, there is a universal constant $C_u >0$, such that
    \begin{align*}
        \int_{\mft} ^ T \int_{\mathbb{R}^d}|\partial^2 _{11} (q_t \delta_t ^1) (x)| \ \de x \ \de t \leq C_u \int_{\mft} ^T \bigg(L_{T-t} +\frac{D^3}{\sigma_t ^6} \bigg) \ \de t \leq C_u (\mathcal{L}+ T \cdot {\mft}^{-2} \cdot D^3).
    \end{align*}
Then, by \eqref{e: error after gagliardo} and Assumption~\ref{a:score-estimate}, 

    \begin{align*}
        \begin{split}
            & \quad \sum_{i=1} ^d \int_{\mft} ^ T{\bigg(\int_{\mathbb{R}^d}|\partial_i (q_t \delta_t ^i) (x)| \ \de x\bigg)} \ \de t \\ &\leq C_u \cdot {(\mathcal{L}+ T \cdot {\mft}^{-2} \cdot D^3)}^{\frac{1}{2}} \cdot \sum_{i=1} ^d {\bigg(\int_{\mft} ^T \int_{\mathbb{R}^d}q_t(x) {(\delta_t ^i(x))}^2 \ \de x \ \de t \bigg) ^{\frac{1}{4}}  } \cdot {(T-\mft)}^{\frac{1}{4}}
            \\  &\leq C_u \cdot {(\mathcal{L}+ T \cdot {\mft}^{-2} \cdot D^3)}^{\frac{1}{2}} \cdot d^{\frac{3}{4}} \cdot \delta^{\frac{1}{2}} \cdot T^{\frac{1}{4}}.
        \end{split}
    \end{align*}
\begin{remark}\label{remark: weighted error}
    We notice that one can also modify the inequality \eqref{e: error after gagliardo} by 
    \begin{align*}
            \int_{\mft} ^ T{\bigg(\int_{\mathbb{R}^d}|\partial_1 (q_t \delta_t ^1) (x)| \ \de x\bigg)} \de t 
            &\leq C_u ^{\frac{1}{2}} {\bigg(\int_{\mft} ^ T \varphi(t) \int_{\mathbb{R}^d}|\partial^2 _{11} (q_t \delta_t ^1) (x)| \ \de x \de t \bigg)}^{\frac{1}{2}}\\
            &\times {\bigg(\int_{\mft} ^T \frac{1}{\varphi(t)} \int_{\mathbb{R}^d}|(q_t \delta_t ^1) (x)| \ \de x \de t\bigg)}^{\frac{1}{2}},
    \end{align*}
    with a suitably chosen positive function $\varphi(t)$ when we use the H\"{o}lder inequality. For example, we can let $\varphi(t) \to 0$ of order $t^2$ as $t \to 0^+$, and let $\varphi(t) \to 0$ of order $\frac{1}{t}$ as $t \to +\infty$. In this way, the first term in the above inequality is uniformly bounded so that we can pass $\mft \to 0^+$ and $T \to +\infty$. So, we only need to control the second term so that it is small enough. 
\end{remark}

\begin{remark}\label{remark: higher order Gagliardo}
    In our settings, if we know that $s_t(x)$ has higher order derivatives up to $k$ for $k \geq 2$, we can replace \eqref{e: Gagliardo k=2} with \eqref{e: Gagliardo general k} when we estimate $E(t)$ in \eqref{e: error after gagliardo}. Then, in order to estimate the derivatives of $q_t$ in the expansion of $\partial^k _{1} (q_t \delta_t)$, a similar proof of Lemma~\ref{Lemma: Hessian estimates} should work because $q_t(x)$ has an exponential tail as $\|x \|_2 \to +\infty$.
\end{remark}

\begin{remark}\label{remark: other assumptions on initial data}
    As readers have seen these proofs under the assumption that $\muast$ has a compact support $K_{\ast}$, the main reason we need this compact support assumption is to estimate the term \begin{align*}
    \int_{\mft} ^ T \int_{\mathbb{R}^d}|\partial^2 _{11} (q_t \delta_t ^1) (x)| \ \de x \ \de t .
    \end{align*}
    Notice that $q_t \delta_t = q_t(s_t - \nabla\log q_t ) = q_t s_t - \nabla q_t$. If we replace the assumption with $\muast$ being a Gaussian mixture, according to Lemma~\ref{lemma: Hessian estimates of Gaussian mixture initial data}, we can do a similar estimate on $|\partial^2 _{11} (q_t \delta_t ^1) (x)|$, and hence we can similarly obtain an upper bound for $\int_{\mft} ^ T \int_{\mathbb{R}^d}|\partial^2 _{11} (q_t \delta_t ^1) (x)| \ \de x \ \de t$. Such an upper bound will then depend on those parameters in the initial Gaussian mixture $\muast$, but don't blow up as $\mft \to 0^+$. See Lemma~\ref{lemma: Hessian estimates of Gaussian mixture initial data}. One can make other reasonable assumptions on $\muast$ as long as one can reasonably estimate this second derivative integral.
\end{remark}

\section{Discretization Error}\label{s:discretize_error}

As discussed in \Cref{s:time_integrator}, we solve the ODE flow  $\del_t \sY_t=\sY_t+s_t(\sY_t)=\sV_t(\sY_t)$ using the Runge-Kutta method. Although the Runge-Kutta updating rule as in \eqref{eq:RK-update} is a discrete time process, we can interpolate it as a continuous time process on $t_i\leq t\leq t_{i+1}$ as 
\begin{align*}
\widetilde Y_{t}=F_r(\widetilde Y_{t_i}),\quad t=t_i+r,
\end{align*}
where
\begin{align}\label{e:Phitmap0}
 F_r(x)= x + r \sum_{j=1}^s b_j k_j,\quad 0\leq r\leq t_{i+1}-t_i=h,
\end{align}
and
\begin{align}\begin{split}\label{e:RK2}
    &k_1(x) = \widehat V_{t_i + r c_1 }(x),\\
    &k_2(x) = \widehat  V_{t_i + r c_2 }\bigl(x + r(a_{21}k_1) \bigr),\\
    &k_3(x) = \widehat V_{t_i + r c_3 }\bigl(x + r (a_{31}k_1 + a_{32}k_2) \bigr),\\
    &\qquad\qquad\qquad\vdots
    \\
    &k_s(x) = \widehat V_{t_i + r c_s }\bigl(x + r (a_{s1}k_1 + a_{s2}k_2 + \cdots + a_{s,s-1}k_{s-1}) \bigr).
\end{split}\end{align}
We denote the density of $\widetilde Y_t$ as $\widetilde \varrho_t$, then at times $t_i$, $\widetilde \varrho_{t_i}$ is the density of $Y_i$ as given by the $i$-th step of the Runge-Kutta methods.
If $ \widehat V_t(\cdot)$ is differentiable in $t$ on $[t_i, t_{i+1}]$, one can see from the above construction $\widetilde Y_t$ is differentiable in $t$, and 
\begin{align}\label{e:dttY}
    \del_t \widetilde Y_t =\del_r F_r(\widetilde Y_{t_i}),\quad t=t_i+r.
\end{align}

The following proposition states that for $t_{i+1}-t_i= h$ small enough, we can rewrite \eqref{e:dttY} as an ODE flow \begin{align}\label{e:discreteODE}
   \del_t \widetilde Y_t=\widetilde V_t(\widetilde Y_t), 
\end{align} 
for $t_i\leq t\leq t_{i+1}$, and $\widetilde V_t(\cdot)$ is close to $V_t(\cdot)$ up to an error of size $\cO(h^p)$. %For notation purpose, we also define $\widetilde s_t(x) = \dV_t(x) - x$.

\begin{proposition}\label{p:high_order_error}
Adopt \Cref{a:score-high-derivative}, and denote $B:=1+\max_{1\leq i<j\leq s}|a_{ji}|+\max_{1\leq j\leq s}|b_j|$ There exists a large constant $C(p,s, B)$, such that if $8sBhdL\leq 1$, then the following holds. For any $0\leq r\leq h$, $F_r$
    is a diffeomorphism from $\bR^d$ to $\bR^d$. We denote its functional inverse as $F^{-1}_{r}(x)$, then 
    \begin{align}\label{e:def_wtV}
      \widetilde V_{t_i+r}(x)=  (\del_r F_r)(F^{-1}_{r}(x)).
    \end{align}
Moreover, for $t_i\leq t\leq t_{i+1}$, $\del_t \widetilde Y_{t}=\widetilde V_{t}(\widetilde Y_{t})$, and 
    \begin{align}\label{e:stdiff}
        \|\widetilde V_{t}(x)- \widehat V_{t}(x)\|_{\infty},\quad  \|\nabla (\widetilde V_{t}(x)- \widehat V_{t}(x))\|_{\infty} \leq C(p,s,B) \cdot L((\sqrt d+\|x\|_2)\sqrt d hL)^p,\quad 
    \end{align}
\end{proposition}

\begin{proof}[Proof of \Cref{theorem: main L^1 theorem discretized}]
We assume $8sBhdL\leq 1$, so the assumptions of \Cref{p:high_order_error} hold. Otherwise, the discretization error in  \eqref{e:total_TV} is bigger than $1$, and the statement \eqref{e:total_TV} holds trivially.

To prove \eqref{e:total_TV} we need to analyze the density evolution under \eqref{e:discreteODE}. We let $\widetilde q_{t}$ 
piecewisely solve the transport equation 
\begin{align*}
    \del_t \widetilde q_{t} = 
    \nabla \cdot (\dU_t
    \widetilde q_{t}) \quad \textrm{ with }\quad \dU_t = \dV_{T-t},
\end{align*}
on each interval $[t_i,t_{i+1}]$ for $i \geq 1$, where $0=t_0<t_1<\cdots<t_{N}=T-\mft$. Then, we define $\widetilde \delta_t(x) \coloneq  \dU_t(x) - U_t(x)$,  $\widetilde \varepsilon_t(x) \coloneq \widetilde q_t(x) - q_t(x)$. We remark that $\dU_t$ is continuous on the $t$-direction when $t \in [t_i,t_{i+1}]$ but it may not be continuous crossing each $t_i$. 
    \begin{align*}
        \begin{split}
            \bigg| \int_{\mathbb{R}^d}  | \widetilde \varepsilon_{\mft}(x)  | \ \de x  - \int_{\mathbb{R}^d}  | \widetilde \varepsilon_{T}(x)  | \ \de x \bigg | &\leq \sum_{i=0} ^{N-1} \bigg| \int_{\mathbb{R}^d}  | \widetilde \varepsilon_{T-t_i}(x)  | \ \de x  - \int_{\mathbb{R}^d}  | \widetilde \varepsilon_{T-t_{i+1}}(x)  | \ \de x \bigg |  \\ &\leq \sum_{i=0} ^{N-1} \int_{T-t_{i+1}} ^{T-t_i}  \int_{\mathbb{R}^d}  \bigg|  (\nabla \cdot (q_{t} \widetilde \delta_{t} ))(x) \bigg| \ \de x \ \de t,
        \end{split}
    \end{align*}
where we used Theorem~\ref{theorem: L^1 error} on each interval $[t_i, t_{i+1}]$.
We also notice that 
    \begin{align}\begin{split}\label{e:triangle}
        \int_{T-t_{i+1}} ^{T-t_i} \int_{\mathbb{R}^d}  \bigg|  (\nabla \cdot (q_{t} \widetilde \delta_{t} ))(x) \bigg| \ \de x \de t 
        &\leq \int_{T-t_{i+1}} ^{T-t_i} \int_{\mathbb{R}^d}  \bigg|  (\nabla \cdot (q_{t} \delta_{t} ))(x) \bigg| \ \de x \de t \\
        &+ \int_{T-t_{i+1}} ^{T-t_i} \int_{\mathbb{R}^d}  \bigg|  (\nabla \cdot (q_{t} (\widetilde \delta_{t}- \delta_{t}) ))(x) \bigg| \ \de x \de t,
    \end{split}
    \end{align}
where the summation of the first term on the righthand side  from $i=0$ to $i=N-1$ is $\int_{\mft} ^T E(t) \de t$ and we have estimated this error term in Section~\ref{section:score estimation error} and also obtained Theorem~\ref{theorem: main L^1 theorem},
\begin{align}\label{e:scorematching}
    \sum_{i=0}^{N-1}\int_{T-t_{i+1}} ^{T-t_i} \int_{\mathbb{R}^d}  \bigg|  (\nabla \cdot (q_{t} \delta_{t} ))(x) \bigg| \ \de x \de t\leq C_u \cdot d^{\frac{3}{4}}  \cdot T^{\frac{1}{4}} \cdot {(TL+ T \cdot {\mft}^{-2} \cdot D^3)}^{\frac{1}{2}} \cdot \delta^{\frac{1}{2}},
\end{align}
where under \Cref{a:score-high-derivative}, $\cL$ in Theorem~\ref{theorem: main L^1 theorem} is bounded by $TL$. This gives the score matching error in \Cref{theorem: main L^1 theorem discretized}.

The second term on the righthand side of \eqref{e:triangle} is 
    \begin{align*}
        \int_{T-t_{i+1}} ^{T-t_i} \int_{\mathbb{R}^d}  \bigg|  (\nabla \cdot (q_{t} (\dU_{t}- \sU_{t}) ))(x) \bigg| \ \de x \de t,
    \end{align*}
which can be further bounded as (the term corresponding to the derivative  $\del_1$)
    \begin{align}
     \begin{split}\label{e:Ttiterm}
            \int_{T-t_{i+1}} ^{T-t_i} \int_{\mathbb{R}^d}  \bigg|  (\partial_1(q_{t} (\dU_{t} ^1 - \sU_{t}^1) ))(x) \bigg| \ \de x \de t \leq &\int_{T-t_{i+1}} ^{T-t_i} \int_{\mathbb{R}^d} \|\dU_{t}(x)-\sU_{t}(x)\|_{\infty} \cdot |\partial_1 q_t(x)| \ \de x \de t
            \\  &+ \int_{T-t_{i+1}} ^{T-t_i} \int_{\mathbb{R}^d} \|\nabla \dU_t(x)-\nabla \sU_t(x)\|_{\infty} \cdot q_t(x) \ \de x \de t,     
     \end{split}
\end{align}
where we use the notation $\dU_t(x) = (\dU_{t} ^1(x), \dU_{t} ^2 (x), \ldots, \dU_{t} ^d(x))$.
By \eqref{e:stdiff}, we obtain that
 \begin{align*}
        \|\widetilde V_{t}(x)-\widehat V_{t}(x)\|_{\infty},\quad  \|\nabla (\widetilde V_{t}(x)-\widehat V_{t}(x))\|_{\infty} \leq C(p,s,B) \cdot L((\sqrt d+\|x\|_2)\sqrt d hL)^p,\quad 
    \end{align*}
    and the definition that $\dU_t = \dV_{T-t}$, $\sU_t = \sV_{T-t}$, we know that the right hand side of \eqref{e:Ttiterm} can be bounded by 
    \begin{align}\label{e:discretization_error}
        C(p,s,B) {L}^{p+1} {(h\sqrt d)}^p \int_{T-t_{i+1}} ^{T-t_i} \int_{\mathbb{R}^d} {(\sqrt d+\|x\|_2)}^p\cdot (|\partial_1 q_t(x)| + q_t(x)) \ \de x \de t.
    \end{align}
According to Lemma~\ref{Lemma: Hessian estimates}, the integral can be bounded by 
    \begin{align}\label{e:moment0qt}
        \int_{T-t_{i+1}} ^{T-t_i} \int_{\mathbb{R}^d} (\sqrt d+\|x\|_2 ^p)\cdot \bigg(\frac{(|x_1| + \lambda_t  D) }{\sigma_t ^2} +1 \bigg) \cdot  q_t(x) \ \de x \de t.
    \end{align}
The above integral can be bounded by using the following two relations. 
\begin{align}\label{e:moment1q_t}
        \begin{split}
            &\phantom{{}={}}\int_{\mathbb{R}^d} |y_1|(\sqrt d+\|y\|_2)^p \cdot q_t(x) \ \de x = \int_{\mathbb{R}^d} \int_{\mathbb{R}^d} \frac{|y_1|(\sqrt d+\|y\|_2)^p}{{(\sqrt{2\pi} \sigma_t)}^{d}} \cdot e^{-\frac{\|y-x\|_2 ^2}{2{\sigma_t}^2}} \cdot \muast\bigg(\frac{x}{\lambda_t}\bigg) \frac{1}{{\lambda_t}^d} \ \de x \de y
            \\  &\leq \int_{\mathbb{R}^d} \int_{\mathbb{R}^d} \frac{({|y_1-x_1|} + {|x_1|})(\sqrt d +\|y-x\|_2+\|x\|_2)^p}{{(\sqrt{2\pi} \sigma_t)}^{d}} \cdot e^{-\frac{\|y-x\|_2 ^2}{2{\sigma_t}^2}} \cdot \muast\bigg(\frac{x}{\lambda_t}\bigg) \frac{1}{{\lambda_t}^d} \ \de x \de y
            \\  &\leq \int_{\mathbb{R}^d} \int_{\mathbb{R}^d} \frac{({|y_1-x_1|} +\lambda_t D)(\sqrt d(1+\lambda_t D)+\|y-x\|_2)^p }{{(\sqrt{2\pi} \sigma_t)}^{d}} \cdot e^{-\frac{\|y-x\|_2^2}{2{\sigma_t}^2}} \cdot \muast\bigg(\frac{x}{\lambda_t}\bigg) \frac{1}{{\lambda_t}^d} \ \de x \de y
            \\  &\leq \int_{\mathbb{R}^d}  \frac{({|z_1|} +\lambda_t D)(\sqrt d(1+\lambda_t D)+\|z\|_2)^p }{{(\sqrt{2\pi} \sigma_t)}^{d}} \cdot e^{-\frac{\|z\|_2 ^2}{2{\sigma_t}^2}} \rd z
            \\  &\leq C(p)d^{p/2} (\sigma_t  + \lambda_t  D)((1+\lambda_t D + \sigma_t)^p) \leq 4^pC(p)d^{p/2} D^{p+1},
        \end{split}
    \end{align}
 where we used that $D\geq 1\geq \sigma_t$ and $\la_t\leq 1$; and similarly
\begin{align}\label{e:moment2q_t}
        \begin{split}
            &\phantom{{}={}}\int_{\mathbb{R}^d} (\sqrt d+\|y\|_2)^p q_t(x) \ \de x 
            \leq 4^pC(p)d^{p/2} D^{p},
        \end{split}
    \end{align}
 where $C(p)$ is constant depending only on $p$.
Combine these above estimates \eqref{e:moment1q_t} and \eqref{e:moment2q_t}, we see that
\begin{align}\label{e:moment3qt}
    \eqref{e:moment0qt}\leq 4^{p+1}\int_{T-t_{i+1}}^{T-t_i}\frac{C(p) d^{p/2}D^{p+1}}{\sigma_t^2} \rd t.
\end{align}
Finally by plugging \eqref{e:moment0qt} and \eqref{e:moment3qt} back into \eqref{e:discretization_error}, we conclude 
    \begin{align}
        \begin{split}\label{e:discretizationbb}
            \sum_{i=0} ^{N-1}  \int_{T-t_{i+1}} ^{T-t_i} \int_{\mathbb{R}^d}  \bigg|  (\nabla \cdot (q_{t} (\dU_{t}- \sU_{t}) ))(x) \bigg| \ \de x \de t &\leq C(p,s,B) {L}^{p+1} {(hd)}^p D^{p+1}\cdot d\int_{\mft} ^{T} \frac{1}{\sigma_t^2}   \ \de t
            \\ &\leq C(p,s,B) d {(hd)}^p (LD)^{p+1}   \log (T/ \mft).
        \end{split}
    \end{align}
This gives the discretization error in \Cref{theorem: main L^1 theorem discretized}. \Cref{theorem: main L^1 theorem discretized} follows from combining \eqref{e:scorematching} and \eqref{e:discretizationbb}.
\end{proof}

\subsection{Proof of \Cref{p:high_order_error}}

We first state some estimates on $F_r(x)$ from \eqref{e:Phitmap0} and $k_j(x)$ from \eqref{e:RK2}.
For a vector $v \in \mathbb{R}^d$, we denote its $q$-th coordinate as $v^{(q)}$, and $\del_p$ is the derivative with respect to $x^{(p)}$.
\begin{lemma}\label{lemma: infinity norm pth order growth in RK}
   Adopt the assumptions of \Cref{p:high_order_error}. For those $k_j(x)$'s as in \eqref{e:RK2}, we define 
        \begin{align*}
            \begin{split}
                &\|\nabla   k_j(x)\|_{\infty} = \sup_{1\leq p,q \leq d} |\partial_p   k_j ^{(q)} (x)|. %\ \|\nabla ^2   k_j(x)\|_{\infty} = \sup_{1\leq p,q,r \leq d} |\partial_{pr}^2   k_j ^{(q)} (x)|.
            \end{split}
        \end{align*}
        We have that, for any $x\in \mathbb{R}^d$, 
        
         \begin{align}\label{e:K2}
            \begin{split}
                &\sum_{j=1} ^s \| k_j(x) \|_{2} \leq s[\|x\|_2+L(\sqrt d+\|x\|_2)]{(1+ L Br)}^{s-1}, \quad \|k_j(x)\|_2\leq 2 L(\sqrt{d}+\|x\|_2).  
            \end{split}
        \end{align} 
        and
        \begin{align}\label{e:Dk}
                &\sum_{j=1} ^s \|\nabla   k_j(x) \|_{\infty} \leq s   L {(1+   L  Br d)}^{s-1} ,\quad \|\nabla   k_j(x)\|_{\infty}\leq  2L.
        \end{align}
 %   and 
 %        \begin{align}
 %           \sum_{j=1} ^s \|\nabla ^2   k_j(x) \|_{\infty} \leq 2   L s^3 {(1+   L B r d)}^{2s} ,
 %       \end{align}
 
\end{lemma}

\begin{lemma}\label{lemma:V derivative bound}
Adopt the assumptions of \Cref{p:high_order_error}.  There exists a large constant $C(p,s,B)$, such that
    \begin{align}\label{e:dtVbound}
        \|\del_r^p \widetilde V_{t_i+r}(F_r(x))\|_\infty, \ 
         \|\del_r^p\nabla \widetilde V_{t_i+r}(F_r(x))\|_\infty\leq C(p,s,B)  L((\sqrt d+\|x\|_2)\sqrt d L)^p,
    \end{align}
and 
\begin{align}\label{e:dhVbound}
        \|\del_r^p \widehat V_{t_i+r}(F_r(x))\|_\infty , \ 
         \|\del_r^p\nabla \widehat V_{t_i+r}(F_r(x))\|_\infty\leq C(p,s,B)   L((\sqrt d+\|x\|_2)\sqrt d L)^p.
    \end{align}
\end{lemma}

In the following, we first prove \Cref{p:high_order_error}. The proofs of \Cref{lemma: infinity norm pth order growth in RK} and \Cref{lemma:V derivative bound} are postponed to the end of this section. 
\begin{proof}[Proof of \Cref{p:high_order_error}]
Recall $F_r(x)=x+r\sum_{j=1}^s b_j k_j$. It follows from \Cref{lemma: infinity norm pth order growth in RK} that 
\begin{align}\begin{split}\label{e:Fr_bound}
    \|F_r(x)-x\|_2&\leq r B \sum_{i=1}^s \|k_i(x)\|_2\leq rBs[\|x\|_2+L(\sqrt d+\|x\|_2)]{(1+ L Br)}^{s-1} \\
    &\leq 2 L rBs{(1+ L Br)}^{s-1}(\sqrt d+\|x\|_2) \leq \frac{1}{2}(\sqrt d+\|x\|_2),
\end{split}\end{align}
where we used that $8sBhL\leq 8sBhdL\leq 1$, and $2 L rBs{(1+ L Br)}^{s-1}\leq 2sLBre^{sLBr}\leq (1/4)e^{1/8}\leq 1/2$. It follows that 
\begin{align}\label{e:change}
    \sqrt d+\|F_r(x)\|_2\geq \sqrt d+\|x\|_2-\|F_r(x)-x\|_2\geq \frac{1}{2}(\sqrt d+\|x\|_2).
\end{align}

Next, to show the map $F_r(x)$ from \eqref{e:Phitmap0} is a local diffeomorphism, we check its Jacobian matrix 
    \begin{align}\label{e:Jacobian}
        \nabla F_r(x)={\mathbb I}_d +A,\quad A:=r \sum_{j=1}^s b_j \nabla k_j(x).
    \end{align}
   Thanks to \Cref{lemma: infinity norm pth order growth in RK}, the $(i,j)$-th entry of $A$ is bounded by
    \begin{align}\label{e:Aijbound}
        |A_{ij}|\leq r B \sum_{i=1}^s \|\nabla k_i(x)\|_\infty\leq rB s L(1+LBr d )^{s-1} ,\quad 1\leq i,j\leq d.
    \end{align}
    The spectral norm $\|A\|_{\rm norm}$ of the matrix $A$ is bounded by its  Frobenius norm as
    \begin{align*}
        \|A\|_{\rm norm}\leq \|A\|_{\rm F}\leq \sqrt{\sum_{ij}A_{ij}^2}\leq rB s L d (1+LBr d )^{s-1}\leq 1/2.
    \end{align*}
    where again we used that $8sLBrd\leq 1$.
    
      It follows that $\nabla F_r(x)=\mathbb I_d+A$ is invertible, and $F_r$ is a local diffeomorphism, and then Hadamard-Cacciopoli theorem implies that $F_r$ is also a bijection from $\bR^d$ to itself. Therefore, $F_r$ is a diffeomorphism from $\bR^d$ to itself. Moreover, thanks to \eqref{e:Aijbound}, we have the following entrywise bound for the inverse matrix $(\nabla F_r(x))^{-1}$:
\begin{align}\begin{split}\label{e:DFr_inverse}
       |((\nabla F_r(x))^{-1}-\mathbb I_d)_{ij}|
       &= |(\mathbb I_d +A)_{ij}^{-1}-\delta_{ij}|\leq |(\mathbb I_d
       -A+A^2-A^3+\cdots)_{ij}-\delta_{ij}|\\
       &\leq \sum_{k\geq 1}(rB s L(1+LBrd )^{s-1})^k d^{k-1}\leq 2rB s L(1+LBrd )^{s-1},\quad 1\leq i,j\leq d,
    \end{split}\end{align}
   where we used that $8sLBrd\leq 1$.
  We denote the functional inverse of $F_r$ as $F^{-1}_{r}(x)$, then \eqref{e:def_wtV} follows from \eqref{e:dttY}.

 %   \begin{align}
 %       \|DV _{t_i}(x)\|\leq \|DV _{t_i}(x)\|_{\rm F} \leq \sqrt{(1+L)^2 d+L^2 d^2}\leq 2Ld.
 %   \end{align}
 %   and we have the same estimate for $V_{t_i+r}$. Thus the Jacobian matrix \eqref{e:Jacobian} is close to the identity matrix
%    \begin{align}
 %       \| DF_r(x)-\mathbb I_d\|
 %       \leq \frac{r}{2} \|D V_{t_i}\|
 %(x+r V_{t_i}(x))\|\|{\mathbb I}_d+r D V_{t_i}(x)\|\leq 2Lrd(1+Lrd)\leq \frac{1}{2}.
%    \end{align}
%    provided that $Lrd\leq 1/8$.

 Next we show the claim \eqref{e:stdiff} follows from the following statement: for $0\leq r\leq t_{i+1}-t_i\leq h$
\begin{align}\begin{split}\label{e:tV-V}
     &\|\widetilde V_{t_i+r}(F_r(x))-\widehat V_{t_i+r}(F_r(x))\|_{\infty}\leq C(p,s, B) L((\sqrt d+\|x\|_2)\sqrt d r L)^p,\\
     &\|\nabla(\widetilde V_{t_i+r}(F_r(x))-\widehat V_{t_i+r}(F_r(x)))\|_{\infty}\leq C(p,s, B) L((\sqrt d+\|x\|_2)\sqrt d  rL)^p.
 \end{split}\end{align}

In fact, if we denote $y=F_r(x)$, then
\begin{align*}
    \|\widetilde V_{t_i+r}(y)-\widehat V_{t_i+r}(y)\|_{\infty}
    &\leq C(p,s,B) L((\sqrt d+\|F^{-1}_{r}(y)\|_2)\sqrt d r L)^p\\
    &\leq 2^{p}C(p,s,B) L((\sqrt d+\|y\|_2)\sqrt d r L)^p,
\end{align*}
where in the last inequality we used \eqref{e:change} to bound $\sqrt d+\|F^{-1}_{r}(y)\|_2$ by $2(\sqrt d+\|y\|_2)$.

For the gradient, by the chain rule we have
\begin{align}\label{e:DeltaV}
    (\nabla \widetilde V_{t_i+r})(y)-(\nabla \widehat V_{t_i+r})(y)=\nabla(\widetilde V_{t_i+r}(F_r(x))-\widehat V_{t_i+r}(F_r(x)))(\nabla F_r(x))^{-1}.
\end{align}
By plugging \eqref{e:DFr_inverse} into \eqref{e:DeltaV}, we conclude that
    \begin{align}\begin{split}\label{e:VFdiff}
        &\phantom{{}={}}\|(\nabla \widetilde V_{t_i+r})(y)-(\nabla \widehat V_{t_i+r})(y)\|_\infty \\
        &\leq (1+2rB s L d (1+LBr d )^{s-1})\|\nabla(\widetilde V_{t_i+r}(F_r(x))-\widehat V_{t_i+r}(F_r(x)))\|_{\infty}\\
    &\leq 2\|\nabla(\widetilde V_{t_i+r}(F_r(x))-\widehat V_{t_i+r}(F_r(x)))\|_{\infty}
    \leq 2C(p,s,B) L((\sqrt d+\|x\|_2)\sqrt d r L)^p,
    %\\&\leq 2^{p+1}C(p,s,B) L((\sqrt d+\|x\|_2)\sqrt d r L)^p,
    \end{split}\end{align}
  where in the third line we used that $8rsBLd\leq 8hsBLd\leq  1$ and \eqref{e:tV-V}.

In the rest we prove \eqref{e:tV-V}. 
We denote the characteristic flow corresponding to $\widehat V_{t_i+r}$ as $\Phi_r(x)$, i.e. $\Phi_0(x)=x$ and $\del_r\Phi_r(x)= \widehat V_{t_i+r}(\Phi_r(x))$. 
Recall from \Cref{r:RK}, the  Runge–Kutta matrix $[a_{jk}]$, weights $b_j$ and nodes $c_j$ are carefully chosen such $\Phi_r(x)$ and $F_r(x)$ matches for the first $p$-th derivative at $r=0$. It follows that for any $0\leq m\leq p$,
\begin{align*}
    \left.\frac{\rd^m \Phi_r(x)}{\rd^m r}\right|_{r=0}=\left.\frac{\rd^m F_r(x)}{\rd^m r}\right|_{r=0},\quad 
    \left.\frac{\rd^m \nabla \Phi_r(x)}{\rd^m r}\right|_{r=0}=\left.\frac{\rd^m \nabla F_r(x)}{\rd^m r}\right|_{r=0}.
\end{align*}
Thus by the chain rule, we have for $0\leq m\leq p-1$,
\begin{align*}
    &\left.\frac{\rd^m \widehat V_{t_i+r}(\Phi_r(x))}{\rd^m r}\right|_{r=0}=\left.\frac{\rd^m \widehat V_{t_i+ r}(F_r(x))}{\rd^m r}\right|_{r=0},\\
    &\left.\frac{\rd^m \nabla \widehat V_{t_i+ r}(\Phi_r(x))}{\rd^m r}\right|_{r=0}=\left.\frac{\rd^m \nabla \widehat V_{t_i+ r}(F_r(x))}{\rd^m r}\right|_{r=0}.
\end{align*}

Then by Taylor expansion we conclude that
\begin{align}\begin{split}\label{e:Taylor}
    &\widehat V_{t_i+r }(F_r(x))-\widetilde V_{t_i+r }(F_r(x))
    =R_1(r,x)-R_2(r,x),\\
    &\nabla \widehat V_{t_i+r }(F_r(x))-\nabla \widetilde V_{t_i+r }(F_r(x))
    =\nabla R_1(r,x)-\nabla R_2(r,x).
\end{split}\end{align}
where the two remainder terms are given by 
\begin{align*}
    &R_1(r,x)=\frac{1}{p!}\int_0^{r}(r-\tau)^{p-1}\frac{\rd^p \widehat V_{t_i+\tau }(F_\tau(x))}{\rd \tau^p} \de \tau,\\
    &R_2(r,x)=\frac{1}{p!}\int_0^{r}(r-\tau)^{p-1}\frac{\rd^p \widetilde V_{t_i+\tau }(F_\tau(x))}{\rd \tau^p} \de \tau .
\end{align*}

We conclude from  \Cref{lemma:V derivative bound} that 
\begin{align}\begin{split}\label{e:gradinfite_bound}
\| R_1(r,x)\|_{\infty},\| R_2(r,x)\|_{\infty}
    \|\nabla R_1(r,x)\|_{\infty},\|\nabla  R_2(r,x)\|_{\infty}
    \leq  C(p,s,B) r^pL((\sqrt d+\|x\|_2 ) \sqrt d  L)^p.
\end{split}\end{align}

The estimates \eqref{e:Taylor} and \eqref{e:gradinfite_bound} together give \eqref{e:tV-V}.
This finishes the proof of \Cref{p:high_order_error}. 
\end{proof}

Next, we are going to prove \Cref{lemma: infinity norm pth order growth in RK}.
\begin{proof}[Proof of \Cref{lemma: infinity norm pth order growth in RK}]
   
To prove \eqref{e:K2},  we notice that, according to our \Cref{a:score-high-derivative}, $\|s_t(x)\|_2\leq L(\sqrt d+\|x\|_2)$ and $\|\widehat V_t(x)\|_2 \leq \|x\|_2+L(\sqrt d+ \|x\|_2)$ for any $x \in \mathbb{R}^d$. Hence, $\|k_1(x)\|_2 \leq \|x\|_2+L(\sqrt d+\|x\|_2)$, according to \eqref{e:RK2}. By the relation that $k_j = \widehat V_{t_i + c_j r}\bigl(x + (a_{j1}    k_1 + a_{j2}    k_2 + \cdots + a_{jj-1}    k_{j-1}) r\bigr)$, we obtain that
        \begin{align*}
            \begin{split}
                &\|   k_j (x)\|_{2} \leq    L\sqrt d+ (1+L)(\|x\|_2+Br (\|   k_1(x)\|_{2} + \|   k_2(x)\|_{2} + \cdots + \|   k_{j-1}(x)\|_{2})),
            \end{split}
        \end{align*}
    If we define $T_m \coloneq \sum_{j=1} ^m \|   k_j(x) \|_{2}$, we see that
    \begin{align*}
        \begin{split}
            &T_j \leq   L \sqrt d+ (1+L)\|x\|_2 +   (1+(1+L)Br )T_{j-1}.
        \end{split}
    \end{align*} 
    Hence,
    \begin{align*}
        \begin{split}
            &T_s \leq s(\|x\|_2+L(\sqrt d+\|x\|_2){(1+ L Br)}^{s-1} , 
        \end{split}
    \end{align*}
    and
    \begin{align*}
        \|k_j(x)\|_2
        &\leq (L\sqrt d+(1+L)\|x\|_2)+(1+L)Brs{(1+ L Br)}^{s-1}(\|x\|_2+L(\sqrt d+\|x\|_2)\\
        &\leq 2L(\sqrt d+\|x\|_2),
    \end{align*}
       where we used $8sLBrd\leq 8sLBhd\leq 1$. This finishes the proof of \eqref{e:K2}.
       
    By the definition that $  k_j(x) = \widehat V_{t_i + r c_j}\bigl(x + (a_{j1}   k_1 + a_{j2}   k_2 + \cdots + a_{jj-1}   k_{j-1}) r\bigr)$, we obtain that
        \begin{align*}
    \partial_{p_1}   k_j ^{(q)} = \sum_{p_2 = 1} ^{d} (\partial_{p_2}\dV_{t_i + c_j r} ^{(q)})[\delta_{p_1 p_2} + r(a_{j1} \partial_{p_1}  k_1 ^{(p_2)} + a_{j2} \partial_{p_1}   k_2 ^{(p_2)} + \cdots + a_{jj-1} \partial_{p_1}   k_{j-1} ^{(p_2)}) ].
        \end{align*}
%and 
%    \begin{align}
%        \begin{split}
%            \partial_{p_1 p_3} ^2   k_j ^{(q)} &= \bigg( \sum_{p_2  = 1} ^{d} \sum_{p_4  = 1} ^{d} (\partial_{p_2 p_4} ^2 \dV_{t_i + c_j r} ^{(q)})[\delta_{p_1 p_2} + r(a_{j1} \partial_{p_1}  k_1 ^{(p_2)} + \cdots + a_{jj-1} \partial_{p_1}   k_{j-1} ^{(p_2)}) ]
%            \\ &\quad \cdot [\delta_{p_3 p_4} + r(a_{j1} \partial_{p_3}  k_1 ^{(p_4)} + \cdots + a_{jj-1} \partial_{p_3}   k_{j-1} ^{(p_4)}) ] \bigg)
%           \\   &\quad  + \sum_{p_2 = 1} ^{d} (\partial_{p_2}\dV_{t_i + rc_j} ^{(q)})[ r(a_{j1} \partial_{p_1 p_3} ^2  k_1 ^{(p_2)} \cdots + a_{jj-1} \partial_{p_1 p_3} ^2   k_{j-1} ^{(p_2)})].
%        \end{split}
 %   \end{align}
Hence,
    \begin{align*}
        \|\nabla   k_j(x)\|_{\infty} \leq L [1+ Br d (\|\nabla   k_1(x)\|_{\infty} +\|\nabla   k_2(x)\|_{\infty} +\cdots+ \|\nabla   k_{i-1}(x)\|_{\infty}) ].
    \end{align*}
If we define $\widetilde T_m \coloneq \sum_{j=1} ^{m} \|\nabla   k_j(x)\|_{\infty}$, we see that 
    \begin{align*}
       \widetilde T_j \leq \widetilde T_{j-1}+ L (1+   Br d\widetilde T_{j-1})=L+(1+LBrd)\widetilde T_{j-1}.
    \end{align*}
Also, $\|\nabla   k_1(x)\|_{\infty} \leq   L$.
So by iteration, 
    \begin{align*}
        \widetilde T_s \leq s   L {(1+   L  Br d)}^{s-1},
    \end{align*}
    and 
    \begin{align*}
         \|\nabla   k_j(x)\|_{\infty}\leq    L (1+BsrdL{(1+   L  Br d)}^{s-1})
        \leq 2L,
    \end{align*}
    where we used $8sLBrd\leq 8sLBhd\leq 1$. This finishes the proof of \eqref{e:Dk}.
\begin{comment}
Similarly,
    \begin{align}
        \|\nabla^2   k_j(x)\|_{\infty} \leq   L {(1+ Br d T_{j-1})}^{2} +   L Br d (\|\nabla ^2   k_1(x)\|_{\infty} +\cdots+ \|\nabla  ^2   k_{j-1}(x)\|_{\infty}),
    \end{align}
and then we define $T_m ' \coloneq \sum_{j=1} ^{m} \|\nabla ^2   k_j(x)\|_{\infty}$. We see that
    \begin{align}
        T_j ' \leq L {(1+ Br d T_{j-1})}^{2} + (1+L Br d) {T_{j-1} '}.
    \end{align}
Also, $\|\nabla ^2 k_1(x)\|_{\infty} \leq L$. So, by iteration
    \begin{align}
        T_s ' \leq L \sum_{j=1} ^{s} {(1+ Br d T_{j-1})}^{2} {(1+ L Br d)}^{s-i} \leq 2 L s^3 {(1+ L Br d)}^{2s-2}.
    \end{align}
For the last inequality, we denote $\gamma \coloneq L Br d$. Then, 
    \begin{align}
        \begin{split}
            &\sum_{i=1} ^s {(1+(i-1)\gamma(1+\gamma)^{i-2})}^2 {(1+\gamma)}^{s-i} \leq 2 \sum_{i=1} ^s \big[{(1+\gamma)}^{s-i} + {(i-1)}^2 \gamma^2 {(1+\gamma)}^{s+i-4}\big]
            \\  &\leq 2s\big[{(1+\gamma)}^{s-1} + {(s-1)}^2 \gamma^2 {(1+\gamma)}^{2s-4}\big] \leq 2s {(1+\gamma)}^{2s-4} (1+s^2 \gamma^2) \leq 2 s^3 {(1+\gamma)}^{2s-2} .
        \end{split}
    \end{align}
 \end{comment}

\end{proof}

\begin{proof}[Proof of \eqref{e:dtVbound}]
In this proof, we will use the notation $X\lesssim Y$ if there exists a constant $C$ depending only on $p,s,B$ such that $|X|\leq C Y$. Given a symmetric tensor $T=(T_{j_1 j_2 \cdots j_{\al+1}})\in \bR^{d^{\al+1}}$, and vectors $u_1, u_2, \cdots, u_\al\in \bR^d$, we denote the vector $T[u_1, u_2, \cdots, u_\al]\in \bR^d$ as
\begin{align*}
    (T[u_1, u_2, \cdots, u_\al])_j
    =\sum_{1\leq j_1,j_2,\cdots, j_\al\leq d}T_{j j_1 j_2 \cdots j_\al}u_1^{(j_1)}u_2^{(j_2)}\cdots u_\al^{(j_\al)}.
\end{align*}
For any matrix $A=(A_{j_1j_2})\in \bR^{d\times d}$, we denote the matrix $T[u_1, u_2, \cdots, u_{\al-1}, A]\in \bR^{d\times d}$ as
\begin{align*}
    (T[u_1, u_2, \cdots, u_{\al-1},A])_{jj'}
    =\sum_{1\leq j_1,j_2,\cdots, j_\al\leq d}T_{j j_1 j_2 \cdots j_\al}u_1^{(j_1)}u_2^{(j_2)}\cdots u_{\al-1}^{(j_{\al-1})}A_{j_\al j'}.
\end{align*}

    Recall that $F_r(x)=x+r\sum_{j=1}^s b_j k_j$, so by \eqref{e:discreteODE}, we have
    \begin{align}\label{e:tV=dF}
        \widetilde V_{t_i+r}(F_r(x))=\del_r F_r(x)=\sum_{j=1}^s b_j k_j+r\sum_{j=1}^s b_j \del_r k_j.
    \end{align}

   In the following, we prove by induction that for any $1\leq m\leq p$
   \begin{align}\label{e:kjbound}
      \|\del_r^m k_j\|_\infty, \|\del_r^m\nabla  k_j\|_\infty\lesssim  L(2(\sqrt d+\|x\|_2) d L)^m.
   \end{align}
    Then claim \eqref{e:dtVbound} follows from plugging \eqref{e:kjbound} to \eqref{e:tV=dF}.
   
   Since $k_1=\widehat V_{t_i+r c_1}(x)$, by \Cref{a:score-high-derivative}, \eqref{e:kjbound} holds for $j=1$ and any $1\leq m\leq p$. In the following we assume statement \eqref{e:kjbound} holds for $j-1$, we prove it for $j$ by induction on $m$

   We define the following set of vectors $\cD_0, \cD_1,\cD_2, \cdots$. Let $\cD_0=\{k_1, k_2, \cdots, k_j\}$. Then thanks to \Cref{lemma: infinity norm pth order growth in RK}, for any $v\in \cD_0$,
   \begin{align}
       \|v\|_2\leq 2L(\sqrt d+\|x\|_2).
   \end{align}

   For $m\geq 1$, $\cD_m$ is defined as the set of vectors in the following form: for  $\beta\geq 1, 0\leq \zeta,\al\leq \beta\leq m$,
    \begin{align}\label{e:defcDm}
        r^{\zeta}\del_r^{\beta-\al}\nabla^{\al} \widehat V_{t_i+r c_j}(x+r(a_{j1}k_1+a_{j2}k_2+\cdots+a_{jj-1}k_{j-1}))[u_1, u_2, \cdots, u_\al],
    \end{align}
    where $\del_r^{\beta-\al}\nabla^{\al} \widehat V_{t_i+r c_j}\in \bR^{d^{\al+1}}$ is a tensor, and for each $1\leq \gamma\leq \al$, \begin{align*}
        u_\gamma\in \{\del_r^{\ell_\gamma} k_1, \del_r^{\ell_\gamma} k_2, \cdots, \del_r^{\ell_\gamma} k_{j-1} \},
    \end{align*}
    for some $\ell_\gamma\geq 0$. Moreover, there exist positive integers $m_1+m_2+\cdots+m_6=m$, such that
    \begin{align}\label{e:parameter}
        \zeta=m_4-m_1,\quad  \al=m_3+m_4, \quad \sum_{\gamma}\ell_\gamma=m_4+m_5+m_6,
        \quad 
        \sum_{\gamma:\ell_\gamma\geq 1}1=m_4+m_5.
    \end{align}

    Next, we show that for each $v\in \cD_m$, $\del_r v$ is a linear combination of at most $2+2s+p$ terms in $\cD_{m+1}$, with coefficients bounded by $\max\{p,B\}$. Say $v$ is given in \eqref{e:defcDm}, satisfying \eqref{e:parameter}. By the chain rule, there are several cases:
    \begin{enumerate}
    \item If $\del_r$ hits $r^\zeta$, we get
    \begin{align*}
        r^{\zeta-1}\del_r^{\beta-\al}\nabla^{\al} \widehat V_{t_i+r c_j}[u_1, u_2, \cdots, u_\al]\in \cD_{m+1},
    \end{align*}
    with the parameters \eqref{e:parameter} given by $(m_1', m_2',m_3', m_4', m_5',m_6')=(m_1+1, m_2,m_3, m_4, m_5,m_6)$.
    \item If $\del_r$ hits $t_i+rc_j$, we get
    \begin{align*}
        r^\zeta\del_r^{(\beta+1)-\al}\nabla^{\al} \widehat V_{t_i+r c_j}[u_1, u_2, \cdots, u_\al]\in \cD_{m+1},
    \end{align*}
    with the parameters \eqref{e:parameter} given by $(m_1', m_2',m_3', m_4', m_5',m_6')=(m_1, m_2+1,m_3, m_4, m_5,m_6)$.
    
    \item If $\del_r$ hits the factor $r$ in $(x+r(a_{j1}k_1+a_{j2}k_2+\cdots+a_{jj-1}k_{j-1})$, we get
    \begin{align*}
        \sum_{1\leq i\leq j-1}r^{\zeta}a_{ji}\del_r^{\beta-\al}\nabla^{\al+1} \widehat V_{t_i+r c_j}[k_i, u_1, u_2, \cdots, u_\al],
    \end{align*}
    where each summand is in $\cD_{m+1}$, with the parameters \eqref{e:parameter} given by $(m_1', m_2',m_3', m_4', m_5',m_6')=(m_1, m_2,m_3+1, m_4, m_5,m_6)$. 
    \item If $\del_r$ hits $k_i$ in $(x+r(a_{j1}k_1+a_{j2}k_2+\cdots+a_{jj-1}k_{j-1})$ for some $1\leq i\leq j-1$, we get
    \begin{align*}
        r^{\zeta+1} \del_r^{\beta-\al}\nabla^{\al+1} \widehat V_{t_i+r c_j}[\del_r k_i, u_1, u_2, \cdots, u_\al]\in \cD_{m+1},
    \end{align*}
    with the parameters \eqref{e:parameter} given by $(m_1', m_2',m_3', m_4', m_5',m_6')=(m_1, m_2,m_3, m_4+1, m_5,m_6)$.
    
    \item If $\del_r$ hits $u_\gamma$ and $u_\gamma=k_i$ we get
    \begin{align*}
        \del_r^{\beta-\al}\nabla^{\al} \widehat V_{t_i+r c_j}[ u_1, u_2, \cdots, u_{\gamma-1}, \del_r k_i, u_{\gamma+1},\cdots, u_\al]\in \cD_{m+1},
    \end{align*}
    with the parameters \eqref{e:parameter} given by $(m_1', m_2',m_3', m_4', m_5',m_6')=(m_1, m_2,m_3, m_4, m_5+1,m_6)$.
    \item If $\del_r$ hits $u_\gamma$ and $u_\gamma=\del_r^{\ell_\gamma}k_i$ with $\ell_\gamma\geq 1$ we get
    \begin{align*}
        \del_r^{\beta-\al}\nabla^{\al} \widehat V_{t_i+r c_j}[ u_1, u_2, \cdots, u_{\gamma-1}, \del_r^{\ell_\gamma+1}k_i, u_{\gamma+1},\cdots, u_\al]\in \cD_{m+1},
    \end{align*}
      with the parameters \eqref{e:parameter} given by $(m_1', m_2',m_3', m_4', m_5',m_6')=(m_1, m_2,m_3, m_4, m_5,m_6+1)$.
\end{enumerate}
We conclude that for $v\in \cD_m$, $\del_r v$ is a linear combination of finite terms in $\cD_{m+1}$. In particular, $\del_r^m k_j$ is a linear combination of at most $(2+2s+p)^p$ terms in $\cD_m$, with coefficients bounded by $\max\{p,B\}^p$

 In the following, we show the following bound for vectors in $v\in \cD_m$
    \begin{align}\label{e:Dkbound}
        \|v\|_\infty\lesssim
            L(2(\sqrt d+\|x\|_2) \sqrt{d} L)^m, \quad  m\geq 1. 
    \end{align}
Then it follows 
\begin{align*}
    \|\del_r^m k_j\|_{\infty} \lesssim L(2(\sqrt d+\|x\|_2) \sqrt{d} L)^m,
\end{align*}
and the first statement in \eqref{e:kjbound} holds.

 We prove \eqref{e:Dkbound} by induction. We recall that $\|k_i\|_2\leq 2L(\sqrt{d}+\|x\|_2)$ from \Cref{lemma: infinity norm pth order growth in RK}. It follows that $\|k_i\|_1\leq 2\sqrt d L(\sqrt{d}+\|x\|_2)$.
    We assume \eqref{e:Dkbound} holds for $ 1,2,\cdots, m-1$, and next we prove it for $m$.
    \begin{align}\begin{split}\label{e:ubb}
        &\phantom{{}={}}\|  r^{\zeta}\del_r^{\beta-\al}\nabla^{\al} \widehat V_{t_i+r c_j}[u_1, u_2, \cdots, u_\al]\|_\infty
        \leq r^{\zeta} \|  \del_r^{\beta-\al}\nabla^{\al} \widehat V_{t_i+r c_j}\|_\infty\|u_1\|_1\|u_2\|_1\cdots\|u_\al\|_1\\
        &\lesssim  L r^{\zeta}  \prod_{\gamma:\ell_\gamma=0}(2(\sqrt d+\|x\|_2)\sqrt d L)
        \prod_{\gamma:\ell_\gamma\geq 1} d L(2(\sqrt d+\|x\|_2)  \sqrt{d} L)^{\ell_\gamma}\\
        &=L r^{\zeta}  (2(\sqrt d+\|x\|_2)\sqrt d L)^\al
        \prod_{\gamma:\ell_\gamma\geq 1} d L(2(\sqrt d+\|x\|_2) \sqrt{d} L)^{\ell_\gamma-1}\\
        &=Lr^{m_4-m_1}(2(\sqrt d+\|x\|_2)\sqrt d L)^{m_3+m_4}( d L)^{m_4+m_5}(2(\sqrt d+\|x\|_2) \sqrt{d} L)^{m_6}
        \\ &
        = L{(r d L)}^{m_4-m_1}(2(\sqrt d+\|x\|_2)\sqrt d L)^{m_3+m_4}( d L)^{m_1+m_5}(2(\sqrt d+\|x\|_2) \sqrt{d} L)^{m_6}
        \\  &
        \lesssim L{(rdL)}^{m_4-m_1}(2(\sqrt d+\|x\|_2) \sqrt{d} L)^{m},
    \end{split}\end{align}
    where in the last line we used \eqref{e:parameter} and $dL \leq (\sqrt d+\|x\|_2) \sqrt{d} L$. 
    By our assumption, $r\sqrt{d}L \leq 1$ and $m_4- m_1 \geq 0$. This finishes the proof of \eqref{e:Dkbound}.

  In the following, we prove the second statement in \eqref{e:kjbound}. We define the following set of $d\times d$ matrices $\widetilde\cD_0, \widetilde\cD_1,\widetilde\cD_2, \cdots$. Let $\widetilde\cD_0=\{\nabla k_1, \nabla  k_2, \cdots, \nabla 
 k_j\}$. Then thanks to \Cref{lemma: infinity norm pth order growth in RK}, for any $v\in \widetilde\cD_0$,
   \begin{align}
       \|v\|_\infty\leq 2L.
   \end{align}
We denote $\widetilde \cD_m$ the set of matrices obtained from taking gradient of \eqref{e:defcDm}, with respect to $x$ 
   \begin{align}\begin{split}\label{e:defcDm2}
        &r^{\zeta}\del_r^{\beta-\al}\nabla^{\al+1} \widehat V_{t_i+r c_j}[(\mathbb{I}_d+r(a_{j1}\nabla k_1+a_{j2}\nabla k_2+\cdots+a_{jj-1}\nabla k_{j-1})), u_1, u_2, \cdots, u_\al]\\
        &r^{\zeta}\del_r^{\beta-\al}\nabla^{\al} \widehat V_{t_i+r c_j}[u_1, u_2, \cdots, u_{\gamma-1},\nabla u_{\gamma}, u_{\gamma+1},\cdots, u_\al], \quad 1\leq \gamma\leq \al.
    \end{split}\end{align}
  Then  $\del_r^m \nabla k_j$ is a finite linear combination of terms in $\widetilde \cD_m$. 

  Next we show by induction that $v\in \cD_m$, 
   \begin{align}\label{e:wDkbound}
        \|v\|_\infty\lesssim
            L(2(\sqrt d+\|x\|_2)\sqrt d L)^m, \quad  m\geq 1. 
    \end{align}
    Then it follows 
\begin{align*}
    \|\del_r^m \nabla k_j\|_{\infty}\lesssim L(2(\sqrt d+\|x\|_2)\sqrt d L)^m,
\end{align*}
and the second statement in \eqref{e:kjbound} holds.

  Similarly to \eqref{e:ubb}, 
  we can bound the $L^\infty$-norm of $v$ as
  \begin{align*}
       L r^\zeta (1+2sBrL) \|u_1\|_1\|u_2\|_1\cdots\|u_\al\|_1
       &\leq 2L r^\zeta \|u_1\|_1\|u_2\|_1\cdots\|u_\al\|_1\\
       &\lesssim 2L(2(\sqrt d+\|x\|_2)\sqrt d L)^m.
    \end{align*}
    Also, for any $q \in \{1, 2,\dots,d\}$, we can similarly obtain that
    \begin{align*}
       L r^\zeta \|\partial_q u_\gamma\|_1 \prod_{\gamma'\neq \gamma}\|u_{\gamma'}\|_1 &\lesssim  L r^\zeta 
       dL(2(\sqrt d+\|x\|_2)\sqrt d L)^{\ell_\gamma}\prod_{\gamma'\neq \gamma}\|u_{\gamma'}\|_1\\ &\lesssim  2L(2(\sqrt d+\|x\|_2)\sqrt d L)^m,
  \end{align*}
as we did in \eqref{e:ubb}.
This gives \eqref{e:wDkbound}.
\end{proof}

\begin{proof}[Proof of \eqref{e:dhVbound}]
    Explicitly, 
    \begin{align*}
        \widehat V_{t_i+r}(F_r)=\widehat V_{t_i+r}(x+r\sum_{j=1}^s b_j k_j),
    \end{align*}
    which is of the same form as $k_j$. Thus the same argument as for \eqref{e:kjbound} gives \eqref{e:dhVbound}
\end{proof}

\end{document}